%% file: main_tmlr_revision.tex
\documentclass[10pt]{article} 
\usepackage[accepted]{tmlr}

\input{math_commands.tex}

\usepackage{xcolor}
\usepackage{hyperref}
\usepackage{url}
\usepackage{todonotes}
\usepackage{booktabs}
\usepackage{graphicx}
\usepackage{tikz-cd}
\usepackage{cleveref}
\usepackage{diagbox}
\usepackage{nicematrix}
\usepackage{amsthm, amssymb}
\usepackage{natbib}
\usepackage{multirow}

\newcommand{\red}[1]{\textcolor{red}{#1}}
\newcommand{\blue}[1]{\textcolor{blue}{#1}}

\newcommand{\rev}[1]{{\color{black}#1}}

\theoremstyle{plain}
\newtheorem{proposition}{Proposition}
\newtheorem{remark}[proposition]{Remark}
\newtheorem{example}[proposition]{Example}

\title{A Discrepancy-Based Perspective on Dataset Condensation}

\author{\name Tong Chen \email toch@di.ku.dk \\
       \addr Department of Computer Science\\
       University of Copenhagen, Denmark
       \AND
       \name Raghavendra Selvan \email raghav@di.ku.dk \\
       \addr Department of Computer Science\\
       University of Copenhagen, Denmark}

\def\month{09}  
\def\year{2026} 
\def\openreview{\url{https://openreview.net/forum?id=xByeYLCMMG}} 

\begin{document}

\maketitle

\begin{abstract}
    \input{abstract}
\end{abstract}

\input{main_content_revision}

\section{Acknowledgement}
\input{acks}

\bibliography{main_revision}
\bibliographystyle{tmlr}

\newpage
\appendix
\input{appendix_revision}

\end{document}

%% file: math_commands.tex
\usepackage{amsmath,amsfonts,bm}

\def\eqref#1{equation~\ref{#1}}

\def\1{\bm{1}}

\DeclareMathAlphabet{\mathsfit}{\encodingdefault}{\sfdefault}{m}{sl}
\SetMathAlphabet{\mathsfit}{bold}{\encodingdefault}{\sfdefault}{bx}{n}



%% file: abstract.tex
Given a dataset of finitely many elements $\mathcal{T} = \{\mathbf{x}_i\}_{i = 1}^N$, the goal of dataset condensation (DC) is to construct a synthetic dataset $\mathcal{S} = \{\tilde{\mathbf{x}}_j\}_{j = 1}^M$ which is significantly smaller ($M \ll N$) such that a model trained from scratch on $\mathcal{S}$ achieves comparable or even superior generalization performance to a model trained on $\mathcal{T}$. Recent advances in DC reveal a close connection to the problem of approximating the data distribution represented by $\mathcal{T}$ with a reduced set of points. In this work, we present a unified framework that encompasses existing DC methods and extend the task-specific notion of DC to a more general and formal definition using notions of discrepancy, which quantify the distance between probability distribution in different regimes. Our framework broadens the objective of DC beyond generalization, accommodating additional objectives such as robustness, privacy, and other desirable properties.

%% file: main_content_revision.tex
\section{Introduction} \label{sec:intro}
Deep learning has achieved remarkable success across a wide range of real-world applications, largely driven by the development of powerful neural network architectures~\citep{lecun2015deep,schmidhuber2015deep}. In contrast to classical machine learning models, modern deep learning systems rely heavily on large and diverse datasets. However, they often suffer from limited interpretability and tend to underperform when trained on smaller datasets. Both theoretical and empirical studies suggest that increasing the size of both the dataset and the model generally leads to improved performance. This insight has motivated practitioners to scale up training datasets and neural network architectures to unprecedented levels in pursuit of higher accuracy~\citep{kaplan2020scaling, hestness2017deep}. However, such scaling comes at a cost: larger models require significantly more energy and contribute to the growing carbon footprint during both training and deployment~\citep{strubell2019energy,anthony2020carbontracker,
patterson2021carbon}.


Rather than solely pursuing ever-larger models, many studies advocate for the opposite direction: regularization and compression. By incorporating suitable regularization terms into the loss function to effectively solve the dual of an optimization problem with appropriate constraints, one can reduce network size while preserving, or even enhancing, downstream performance \citep{srivastava2014dropout, zhang2016understanding, hinton2015distilling, hubara2016quantized, carreira2017model1, carreira2017model2, molchanov2016pruning, jaderberg2014speeding}. 
Most existing work on compression focuses on the model itself, typically keeping the dataset and learning algorithm fixed while modifying model architecture or parameter magnitudes. Notable exceptions include few-shot learning \citep{chen2019closer, brown2020language}, which uses a limited number of training samples combined with more sophisticated learning algorithms, and dataset condensation (DC)~\citep{zhao2023dataset, zhao2021dataset, cazenavette2022dataset}, which aims to reduce dataset size while keeping both the model architecture and learning algorithm unchanged. Despite growing interest in DC, most existing methods are heuristic in nature and heavily reliant on empirical tuning. The landscape of DC techniques remains fragmented, with approaches differing widely in their motivations, mechanisms, and evaluation metrics. This lack of a coherent, principled foundation makes it difficult to systematically compare methods or extend them to new objectives. \textbf{In this work, we aim to unify and formalize the field by proposing a general framework that encompasses existing DC strategies while highlighting the underlying assumptions they share.}

\rev{\paragraph{Related surveys.} Several surveys and reviews of dataset distillation/condensation have appeared in recent years \citep{sachdeva2023data, geng2023survey, lei2023comprehensive, yu2023dataset, liu2025evolution}. These works organize the literature primarily by mechanism, data modality, application area, and chronology, and they document the field's rapid empirical progress. Our contribution is complementary rather than competing: we provide a formal, discrepancy-based account of \emph{what} each family of objectives preserves, \emph{under which} assumptions, and with \emph{which type} of correspondence to the formal objective (exact, estimator, upper bound, relaxation, or analogy; see \Cref{tab:overview}). We intend the present paper as a companion reference to be read alongside these surveys, not as a replacement for them.} 


To formally categorize deep learning tasks, we consider three fundamental components: the dataset (associated with a probability distribution or measure), the model (or hypothesis), and the algorithm. Let $(\Omega, \mathcal{F}, \mu)$ be a probability space, where $\Omega$ is the sample space, $\mathcal{F}$ is a $\sigma$-algebra, and $\mu$ is a probability measure over $\Omega$. When $\Omega$ corresponds to a finite dataset $\Omega = \{\mathbf{x}_i\}_{i = 1}^N$, the empirical measure is defined as $\mu = \frac{1}{N} \sum_{i = 1}^N \delta_{\mathbf{x}_i}$, where $\delta_{\mathbf{x}_i}$ denotes the Dirac measure at $\mathbf{x}_i$. Let $\mathcal{H}$ denote the hypothesis space. A learning algorithm $\mathcal{A}$ maps a probability measure to a hypothesis, i.e., $\mathcal{A}: \mu \mapsto h \in \mathcal{H}$. A criterion $\varphi$ evaluates the quality of a hypothesis under a given measure, defined as $\varphi: (h, \mu) \mapsto v \in \mathbb{R}$. A hypothesis $h$ is said to be optimal with respect to $\mu$ if it minimizes (or maximizes, depending on context) the criterion $\varphi(h, \mu)$. Typically, learning algorithms are designed to find such an optimal hypothesis. 

Many well-known concepts have emerged from modifying one of the three components (dataset, model, or algorithm) while keeping the others fixed. In contrast, jointly altering two or more components poses greater conceptual and technical challenges. Many existing research directions can be interpreted through the lens of holding some components fixed while varying others. Examples include:

\begin{enumerate}
    \item \textbf{Model training}: Given a fixed probability measure $\mu$ and hypothesis space $\mathcal{H}$, design an algorithm $\mathcal{A}$ that returns a model $h$ with optimal performance under criterion $\varphi$, i.e, $\mathcal{A} (\mu) = \arg \min_{\red{h} \in \mathcal{H}} \varphi (\red{h}, \mu)$;
    
    \item \textbf{Model compression}: Again fixing $\mu$ and $\mathcal{H}$, design an algorithm that not only achieves optimal performance but also satisfies a sparsity constraint, i.e., $\mathcal{A} (\mu) = \arg \min_{\red{h} \in \mathcal{H}} \{\varphi (\red{h}, \mu): \|\red{h}\|_0 \le s\}$, \rev{where $s \in \mathbb{N}$ denotes the sparsity budget, i.e., the maximum number of non-zero parameters permitted in the model};
    
    \item \textbf{Dataset compression}: Fixing the hypothesis space $\mathcal{H}$ and algorithm $\mathcal{A}$, find a smaller dataset $\tilde{\mathcal{D}}$ with associated empirical measure $\tilde{\mu}$ such that the model $h^*_{\tilde{\mu}} = \mathcal{A}(\tilde{\mu})$ remains optimal when evaluated on the original distribution $\mu$, that is, $\tilde{\mu} = \arg \min_{\red{\nu}} \varphi (\mathcal{A} (\red{\nu}), \mu)$, where $\mathcal{A} (\nu) = \arg \min_{\red{h} \in \mathcal{H}} \varphi (\red{h}, \nu)$\rev{, subject to the compression budget $|\tilde{\mathcal{D}}| \le M \ll N$, which is made explicit in the formal problem statement of \Cref{sec:dc}}.
\end{enumerate}

In the following sections, we first introduce a unified theoretical framework for dataset compression. We then use this framework to categorize a wide range of existing approaches according to four key dimensions: (1) The metric used to quantify the distance between probability distributions; (2) The feature space in which the dataset is condensed; (3) The parameter space over which models are optimized; and (4) The loss function and regularization terms employed during training.

\begin{table}[ht]
    \caption{Summary of notation used throughout the paper.}
    \label{tab:notation}
    \centering
    \begin{tabular}{ll}
        \toprule
        Notation & Description \\
        \midrule
        $\|\cdot\|$ & Norm (context-dependent: may refer to $\ell_2$, sup-norm, etc.) \\
        $\Omega$ & Input domain or sample space, $\Omega \subseteq \mathbb{R}^n$ \\
        $\mu$ & Probability distribution (measure) over domain $\Omega$ \\
        $\mathcal{P}(\Omega)$ & Space of all probability distributions over $\Omega$ \\
        \rev{$\mathcal{T}, \mathcal{S}$} & \rev{Original and synthetic datasets; $\mathcal{T}, \mathcal{S}$ are reserved for datasets and $D$ for discrepancies} \\
        \rev{$\hat{\mu}_{\mathcal{T}}, \hat{\mu}_{\mathcal{S}}$} & \rev{Empirical distributions induced by $\mathcal{T}$ and $\mathcal{S}$} \\
        $\mathcal{H}(\Omega)$ & Hypothesis space over domain $\Omega$ \\
        $\overline{h(\mathcal{T})}$ & Empirical average of hypothesis $h$ over dataset \rev{$\mathcal{T}$} \\
        $h^*_\mu$ & Optimal hypothesis in $\mathcal{H}$ trained under distribution $\mu$ \\
        $\varphi(h, \mu)$ & Criterion of hypothesis $h$ under distribution $\mu$ \\
        
        $\mathcal{A}$ & Learning algorithm mapping datasets to hypotheses \\
        $\mathcal{C}^n(\Omega)$ & Set of $n$-times continuously differentiable functions over $\Omega$ \\
        $\nabla \mathcal{C}^n(\Omega)$ & Set of gradients of functions in $\mathcal{C}^n(\Omega)$ \\

        $g^{\#}\mu$ & Push-forward measure of $\mu$ via measurable map $g$ \\
        $D(\mu_1, \mu_2; \mathcal{H})$ & \rev{Discrepancy w.r.t.\ test class $\mathcal{H}$ ($\mathrm{GD}, \mathrm{DD}, \mathrm{VD}, \mathrm{PD}, \mathrm{CD}$: specific variants)} \\
        \rev{$\mathcal{Q}_M$} & \rev{Feasible set of measures supported on at most $M$ atoms} \\
        \rev{$\mathcal{F}_{\ell}$} & \rev{Loss-induced test class (\Cref{prop:transfer})} \\
        \rev{$\mathfrak{R}_N(\mathcal{F})$} & \rev{Rademacher complexity of class $\mathcal{F}$ over samples of size $N$} \\
        \bottomrule
    \end{tabular}

\end{table}

\section{Preliminaries and Notations} \label{sec:pre}
Let $\mathbb{R}^n$ denote the $n$-dimensional real vector space. A finite subset $\mathcal{T} = \{\mathbf{x}_i\}_{i = 1}^N \subseteq \mathbb{R}^n$ is referred to as a \emph{dataset}, and naturally induces an empirical distribution \rev{$\hat{\mu} = \frac{1}{N} \sum_{i = 1}^N \delta_{\mathbf{x}_i}$, where $\delta_{\mathbf{x}_i}$ denotes the Dirac measure at $\mathbf{x}_i$}. We will use distributional notation to represent datasets throughout\rev{; we reserve $\mathcal{T}$ (resp.\ $\mathcal{S}$) for original (resp.\ synthetic) datasets and $D$ exclusively for discrepancies between measures, and operator names ($\mathrm{IPM}$, $\mathrm{MMD}$, $\mathrm{GD}$, $\mathrm{DD}$, $\mathrm{VD}$, $\mathrm{PD}$, $\mathrm{CD}$, $\mathrm{Lip}$) are set upright. Our measure-theoretic conventions follow \citet{dudley2002real}, and the notation for integral probability metrics follows \citet{muller1997integral, sriperumbudur2010hilbert}}. Let $\mathcal{P} (\Omega)$ denote the space of probability distributions over $\Omega$. We assume that each sample in $\mathcal{T}$ is drawn i.i.d. from an unknown ground-truth distribution $\mu \in \mathcal{P} (\Omega)$, i.e., $\hat{\mu}$ serves as an empirical approximation of $\mu$. 

To develop learning algorithms, we must first specify the performance criteria under consideration. For standard tasks in deep learning such as classification and regression, let $\ell: \mathbb{R} \times \mathbb{R} \to \mathbb{R}_+$ be a loss function, and let $h \in \mathcal{H} (\Omega)$ denote a model from a given hypothesis space. Two common evaluation criteria are:
\begin{enumerate}
    \item \textbf{Test performance} is defined by the expected loss under the true distribution: $\varphi (h, \mu) = \mathbb{E}_{(\mathbf{x}, y) \sim \mu} [\ell (h(\mathbf{x}), y)]$;

    \item \textbf{Adversarial robustness}, with respect to a perturbation radius $\varepsilon > 0$ and norm $\|\cdot\|$, is measured by the worst-case loss under bounded input perturbations:
    $$\varphi_{\varepsilon} (h, \mu) = \mathbb{E}_{(\mathbf{x}, y) \sim \mu} \bigg[\max_{\|\red{\delta}\| \le \varepsilon} \ell (h (\mathbf{x} + \red{\delta}), y)\bigg],$$
    \rev{where $\delta \in \mathbb{R}^n$ denotes an adversarial perturbation of the input, constrained to the norm ball $\|\delta\| \le \varepsilon$.}
\end{enumerate}

For a fixed distribution $\mu$, the evaluation criterion $\varphi (h, \mu)$ quantifies how well a model $h$ performs on data drawn from $\mu$. In this sense, $\varphi$ serves as a model-dependent statistical measurement of a distribution. When comparing two distributions, a natural approach is to evaluate both using the same criterion and model. However, such a comparison is necessary but not sufficient for distinguishing distributions: two distributions may yield identical performance under one model or criterion yet differ under another. 

To address this limitation, we can define model-independent metrics by either maximizing or averaging the criterion over a hypothesis space $\mathcal{H}$, such as $\varphi (\mu) = \sup_{h \in \mathcal{H}} \varphi (h, \mu)$ or $\mathbb{E}_{h \in \mathcal{H}} [\varphi (h, \mu)]$. We refer to such quantities as \emph{distribution discrepancy (DD)}, which is a criterion that abstracts away the specific model and captures broader statistical properties of the distribution. 

In summary, we categorize statistical metrics for comparing two distributions $\mu_1$ and $\mu_2$ as follows:
\begin{enumerate}
    \item \textbf{Model-specific discrepancy}: $D (\mu_1, \mu_2; h) = |\varphi (h, \mu_1) - \varphi (h, \mu_2)|$, which compares distributions with respect to a fixed model $h$;

    \item \textbf{Model-agnostic discrepancy}: $D (\mu_1, \mu_2; \mathcal{H}) = \sup_{\red{h} \in \mathcal{H}} |\varphi (\red{h}, \mu_1) - \varphi (\red{h}, \mu_2)|$, which captures the maximal difference in performance across all models in the hypothesis space.
\end{enumerate}

It is straightforward that $D (\mu_1, \mu_2; h) = 0$ whenever $\mu_1 = \mu_2$. However, to qualify as a valid statistical distance, a metric must also satisfy the converse: if $D (\mu_1, \mu_2; h) = 0$, then $\mu_1 = \mu_2$. In other words, the metric must be capable of distinguishing any two distinct distributions. Such properties are satisfied by several well-established metrics in statistics, particularly in the context of two-sample testing. Below, we summarize a few canonical examples:
\begin{enumerate}
    \item \textbf{Integral Probability Metric (IPM)}: Let $\varphi (h, \mu) = \mathbb{E}_{(\mathbf{x}, y) \sim \mu} [h (\mathbf{x}, y)]$, and define:
    $$\mathrm{IPM} (\mu_1, \mu_2; \mathcal{H}) = \sup_{\red{h} \in \mathcal{H}} \big|\varphi (\red{h}, \mu_1) - \varphi (\red{h}, \mu_2)\big|.$$
    The IPM forms a valid metric if the function class $\mathcal{H}$ is rich enough to separate distributions. \rev{For the supremum to be finite, $\mathcal{H}$ must be uniformly bounded or norm-constrained: an IPM over \emph{all} measurable functions is infinite whenever $\mu_1 \neq \mu_2$. The canonical bounded choice $\mathcal{H} = \{h \text{ measurable}: \|h\|_{\infty} \le 1\}$ recovers (twice) the total variation distance \citep{muller1997integral}.}

    \item \textbf{Maximum Mean Discrepancy (MMD)}: Assume $\mathcal{H}$ is a reproducing kernel Hilbert space (RKHS), and again define $\varphi (h, \mu) = \mathbb{E}_{(\mathbf{x}, y) \sim \mu} [h (\mathbf{x}, y)]$. Then,
    \begin{align} \label{eq:mmd}
        \mathrm{MMD} (\mu_1, \mu_2; \mathcal{H}) = \sup_{\|\red{h}\|_{\mathcal{H}} \le 1} \big|\varphi (\red{h}, \mu_1) - \varphi (\red{h}, \mu_2)\big|.
    \end{align}
    MMD is a special case of IPM where the function class is the unit ball in an RKHS. With a characteristic kernel, MMD satisfies $\mathrm{MMD} (\mu_1, \mu_2; \mathcal{H}) = 0 \Leftrightarrow \mu_1 = \mu_2$. \rev{Characteristicness is precisely injectivity of the kernel mean embedding $\mu \mapsto \tilde{\mu}$ introduced below; universal kernels on compact metric spaces are characteristic, but universality and characteristicness do not coincide in general \citep{sriperumbudur2010hilbert, sriperumbudur2011universality}.}

    \item \textbf{Wasserstein Metric (WM)}: Let $d$ be a ground distance on $\Omega$, and fix $p \in [1, +\infty]$. The Wasserstein distance of order $p$ is defined as:
    $$\mathrm{WM}_p (\mu_1, \mu_2; \mathcal{H}) = \inf_{\red{\gamma} \in \Gamma (\mu_1, \mu_2)} \big(\mathbb{E}_{((\mathbf{x}_1, y_1), (\mathbf{x}_2, y_2)) \sim \red{\gamma}} [d ((\mathbf{x}_1, y_1), (\mathbf{x}_2, y_2))^p]\big)^{1/p},$$
    where $\Gamma (\mu_1, \mu_2)$ denotes the set of all couplings (joint distributions) with marginals $\mu_1$ and $\mu_2$. In the special case $p = 1$, the Wasserstein metric admits a dual formulation:
    $$\mathrm{WM}_1 (\mu_1, \mu_2; \mathcal{H}) = \sup_{\mathrm{Lip}(\red{h}) \le 1} \big|\varphi (\red{h}, \mu_1) - \varphi (\red{h}, \mu_2)\big|,$$
    where the supremum is taken over all 1-Lipschitz functions $h$. \rev{We emphasize that this Kantorovich--Rubinstein duality is specific to $p = 1$ (and requires $\mu_1, \mu_2$ to have finite first moments): for $p > 1$, $\mathrm{WM}_p$ is defined only through the coupling formulation and is \emph{not} an integral probability metric, and the limiting case $p = \infty$ requires a separate definition via the essential supremum of the coupled cost \citep{villani2009optimal}.}
\end{enumerate}

We observe that both MMD and the WM are special cases of IPM, distinguished by their choice of hypothesis space $\mathcal{H}$. To briefly recap, suppose $\mathcal{H}$ is an RKHS over $\Omega$, equipped with a kernel $k: \Omega \times \Omega \rightarrow \mathbb{R}$, inner product $\langle \cdot, \cdot \rangle_{\mathcal{H}}$, and associated feature maps $\varphi_{\mathbf{x}}: \Omega \rightarrow \mathcal{H}$, then, by the Riesz representation theorem, $\mathcal{H}$ satisfies the reproducing property: for all $h \in \mathcal{H}$ and $\mathbf{x} \in \Omega$, there exists $\varphi_{\mathbf{x}} \in \mathcal{H}$ such that $h (\mathbf{x}) = \langle h, \varphi_{\mathbf{x}} \rangle_{\mathcal{H}}$ and $k (\mathbf{x}_1, \mathbf{x}_2) = \langle \varphi_{\mathbf{x}_1}, \varphi_{\mathbf{x}_2} \rangle_{\mathcal{H}}$. Popular choices of kernels include:
\begin{enumerate}
    \item \textbf{$\gamma$-exponential kernel}: $k^{\gamma} (\mathbf{x}_1, \mathbf{x}_2) := \exp (-c \|\mathbf{x}_1 - \mathbf{x}_2\|^{\gamma})$ with $0 < \gamma \le 2$. Special cases include the Laplacian kernel ($\gamma = 1$) and Gaussian kernel ($\gamma = 2$);

    \item \textbf{Neural tangent kernel (NTK)}: $k^{\theta} (\mathbf{x}_1, \mathbf{x}_2) := \mathbb{E}_{\theta} \big[\nabla f_{\theta} (\mathbf{x}_1)^T \nabla f_{\theta} (\mathbf{x}_2)\big]$, where $f_{\theta}$ is a neural network with parameters $\theta$.
\end{enumerate}

The \emph{kernel embedding} of a probability distribution $\mu$ into an RKHS $\mathcal{H}$ is defined as $\tilde{\mu} := \mathbb{E}_{\mathbf{x} \sim \mu} [k (\mathbf{x}, \cdot)] = \mathbb{E}_{\mathbf{x} \sim \mu} [\varphi_{\mathbf{x}}]$, where $\varphi_{\mathbf{x}}$ is the feature map associated with kernel $k$. This embedding enables expressing the expectation of any function $f \in \mathcal{H}$ as an inner product between $f$ and the kernel embedding, i.e., $\mathbb{E}_{\mathbf{x} \sim \mu} [f(\mathbf{x})] = \langle f, \tilde{\mu} \rangle_{\mathcal{H}}$. In this way, the MMD between two distributions $\mu_1$ and $\mu_2$ defined in \Cref{eq:mmd} can be rewritten as the norm of the difference between their kernel embeddings:
$$\mathrm{MMD} (\mu_1, \mu_2; \mathcal{H}) = \sup_{\|\red{h}\|_{\mathcal{H}} \le 1} \big|\langle \red{h}, \tilde{\mu}_1 - \tilde{\mu}_2 \rangle\big| = \|\tilde{\mu}_1 - \tilde{\mu}_2\|_{\mathcal{H}}.$$
Expanding the squared norm yields a kernel-based expression that avoids explicit use of feature maps:
$$\mathrm{MMD} (\mu_1, \mu_2; \mathcal{H})^2 = \mathbb{E}_{\mathbf{x}_1, \mathbf{x}_1' \sim \mu_1} [k (\mathbf{x}_1, \mathbf{x}_1')] - 2 \cdot \mathbb{E}_{\mathbf{x}_1 \sim \mu_1, \mathbf{x}_2 \sim \mu_2} [k (\mathbf{x}_1, \mathbf{x}_2)] + \mathbb{E}_{\mathbf{x}_2, \mathbf{x}_2' \sim \mu_2} [k (\mathbf{x}_2, \mathbf{x}_2')].$$
This kernel-only formulation is particularly useful for empirical estimation of MMD from samples.

The last important property of distributional metrics is injectivity: under what conditions does a vanishing discrepancy imply that two distributions are identical? Formally, we ask when $D (\mu_1, \mu_2; \mathcal{H}) = 0$ implies $\mu_1 = \mu_2$. For $\Omega \subseteq \mathbb{R}^n$, let $\mathcal{C} (\Omega)$ (resp. $\mathcal{C}^{\infty} (\Omega)$) denote the set of all bounded continuous (resp. bounded smooth) functions over $\Omega$. \rev{We further define the \emph{scalarized gradient class}
$$\nabla \mathcal{C}^{\infty} (\Omega) := \big\{\mathbf{x} \mapsto \langle \mathbf{v}, \nabla f (\mathbf{x}) \rangle : f \in \mathcal{C}^{\infty} (\Omega) \text{ with bounded gradient}, \; \|\mathbf{v}\| \le 1\big\},$$
i.e., the class of (unit) directional derivatives of bounded smooth functions; this places the gradient-based test classes used later within the scalar IPM notation.} We summarize the sufficient conditions for injectivity in the following proposition:
\begin{proposition} \label{prop:discrepancy}
    \rev{Let $\Omega \subseteq \mathbb{R}^n$ be a Borel set equipped with the Euclidean metric, and let $\mu_1, \mu_2 \in \mathcal{P}(\Omega)$.} Let $\mathrm{DD} (\mu_1, \mu_2; \mathcal{H})$ be one of the discrepancy measures previously defined. Then $\mathrm{DD} (\mu_1, \mu_2; \mathcal{H}) = 0$ implies $\mu_1 = \mu_2$ if any of the following holds:

    (1) $\mathrm{DD} (\mu_1, \mu_2; \mathcal{H}) = \mathrm{IPM} (\mu_1, \mu_2; \mathcal{H})$, for $\mathcal{H} = \mathcal{C} (\Omega)$, $\mathcal{C}^{\infty} (\Omega)$, or \rev{(for $\Omega = \mathbb{R}^n$)} $\nabla \mathcal{C}^{\infty} (\Omega)$;

    (2) $\mathrm{DD} (\mu_1, \mu_2; \mathcal{H}) = \mathrm{MMD} (\mu_1, \mu_2; \mathcal{H})$, for $\mathcal{H}$ \rev{the unit ball of} an RKHS with a \rev{bounded, measurable \emph{characteristic} kernel (universality of the kernel is sufficient on compact metric spaces, but not equivalent to characteristicness in general \citep{sriperumbudur2010hilbert, sriperumbudur2011universality})};

    (3) $\mathrm{DD} (\mu_1, \mu_2; \mathcal{H}) = \mathrm{WM}_p (\mu_1, \mu_2)$, for any $p \in [1, + \infty)$ \rev{provided $\mu_1, \mu_2 \in \mathcal{P}_p(\Omega)$, i.e., both have finite $p$-th moments; the case $p = \infty$ additionally requires compactly supported measures}.
\end{proposition}
\begin{proof}
    \rev{We first prove (1). Let $U \subseteq \Omega$ be open in the subspace topology and write $U = V \cap \Omega$ with $V \subseteq \mathbb{R}^n$ open. Borel probability measures on Polish spaces are inner regular with respect to compact sets \citep[Ch.~7]{dudley2002real}, so $\mu_i(U) = \sup\{\mu_i(K): K \subseteq U \text{ compact}\}$. For every compact $K \subseteq V$ there is a smooth Urysohn function $g \in \mathcal{C}^{\infty}(\mathbb{R}^n)$ with $\mathbf{1}_K \le g \le \mathbf{1}_V$ (mollify the indicator of a small neighborhood of $K$); hence $\mu_i(K) \le \int_{\Omega} g \, d\mu_i \le \mu_i(U)$ for $i = 1, 2$, and taking suprema over such $g$ yields $\mu_i(U) = \sup_g \int_{\Omega} g \, d\mu_i$, an approximation \emph{from below in $L^1(\mu_i)$}, which suffices for the argument (note that indicators of open sets admit no \emph{uniform} approximation by continuous functions, since a uniform limit of continuous functions is continuous). If the IPM vanishes over a class containing these Urysohn functions (as $\mathcal{C}(\Omega)$ and $\mathcal{C}^{\infty}(\Omega)$ do), then $\mu_1(U) = \mu_2(U)$ for every open $U$; since the open sets form a $\pi$-system generating the Borel $\sigma$-algebra, a $\pi$--$\lambda$ (Dynkin) argument gives $\mu_1 = \mu_2$. For $\nabla \mathcal{C}^{\infty}(\mathbb{R}^n)$: the bounded smooth functions $f_{\mathbf{t}}(\mathbf{x}) = \sin \langle \mathbf{t}, \mathbf{x} \rangle$ and $\cos \langle \mathbf{t}, \mathbf{x} \rangle$ have bounded gradients, and their unit directional derivatives recover $\cos \langle \mathbf{t}, \cdot \rangle$ and $\sin \langle \mathbf{t}, \cdot \rangle$ up to nonzero scalars whenever $\langle \mathbf{t}, \mathbf{v} \rangle \neq 0$. A vanishing IPM over $\nabla \mathcal{C}^{\infty}(\mathbb{R}^n)$ therefore forces the characteristic functions of $\mu_1$ and $\mu_2$ to agree for all $\mathbf{t} \neq \mathbf{0}$, hence everywhere by continuity, and $\mu_1 = \mu_2$ by the uniqueness theorem for characteristic functions \citep[Ch.~9]{dudley2002real}. Claim (2) holds by the definition of a characteristic kernel: the mean embedding $\mu \mapsto \tilde{\mu}$ is injective and $\mathrm{MMD} (\mu_1, \mu_2; \mathcal{H}) = \|\tilde{\mu}_1 - \tilde{\mu}_2\|_{\mathcal{H}}$ \citep{sriperumbudur2010hilbert}. Claim (3) is the standard fact that $\mathrm{WM}_p$ is a metric on $\mathcal{P}_p(\Omega)$ \citep[Ch.~6]{villani2009optimal}.}
\end{proof}
\rev{The proof shows that uniform density of $\mathcal{H} (\Omega)$ in $\mathcal{C} (\Omega)$ is not required indeed, no sup-norm approximation condition of this kind is attainable. What suffices is the following weaker, measure-dependent richness condition: $\mathcal{H} (\Omega)$ approximates the smooth Urysohn family from the proof in $L^1(\mu_1) \cap L^1(\mu_2)$ for the measures under consideration. Notably, $\mathcal{H} (\Omega)$ need not be a vector space.}

\section{Dataset Condensation} \label{sec:dc}
Most of the DC literature focuses on classification tasks, where the condensation process is typically performed on a per-class basis. That is, for each class label $y$, condensation aims to approximate the conditional distribution of the data given the class. Consequently, the key objective becomes measuring the discrepancy between the class-conditional distributions $\mu_1 (\cdot | y)$ and $\mu_2 (\cdot | y)$. Note that we can also consider soft labels with stochasticity, allowing data and labels to be synthesized jointly. From this perspective, dataset condensation can be naturally generalized to regression tasks. In the following discussion, we adopt this classification convention and assume that $\mu_1$ and $\mu_2$ denote class-conditional distributions. \rev{Two clarifications are in order. First, matching all class-conditional distributions $\mu_1(\cdot | y) \approx \mu_2(\cdot | y)$ recovers the joint distribution only if the class priors are also preserved, i.e., $\mu_2(y) = \mu_1(y)$ for every $y$; we therefore always pair per-class matching with matched class weights. Second, per-class condensation comes with per-class budgets $M_y$, i.e., the ``images-per-class'' (IPC) setting of the empirical literature, with total budget $M = \sum_y M_y$; unless stated otherwise, the statements below hold for any fixed allocation $(M_y)_y$.} All functions $f$ are understood to act on the input variable $\mathbf{x}$, i.e., $f: \mathbb{R}^n \to \mathbb{R}$. In the current literature, DC is typically studied in the context of supervised learning. Given a target data distribution $\mu_1$, the goal is to construct a synthetic distribution $\mu_2$ such that models trained on  $\mu_2$ generalize similarly to those trained on $\mu_1$. Formally, this objective can be formulated as the following bi-level optimization problem\rev{, in which the compression requirement is part of the formal problem: writing
$$\mathcal{Q}_M := \big\{\nu \in \mathcal{P} (\Omega): |\mathrm{supp} (\nu)| \le M\big\}$$
for the feasible set of synthetic measures supported on at most $M$ atoms, we solve}
\begin{align*} \label{eq:dc}
    \inf_{\red{\mu_2} \in \rev{\mathcal{Q}_M}} \; \{\varphi (h^*_{\red{\mu_2}}, \mu_1): h^*_{\red{\mu_2}} \in \arg\min_{\blue{h} \in \mathcal{H}} \; \varphi (\blue{h}, \red{\mu_2}) \}, \tag{DC}
\end{align*}
where $\mathcal{H}$ is a hypothesis space and $\varphi: \mathcal{H} \times \mathcal{P} \to \mathbb{R}$ denotes the expected generalization loss, defined as $\varphi (h, \mu) := \mathbb{E}_{(\mathbf{x}, y) \sim \mu} [\ell (h (\mathbf{x}), y)]$, for some loss function $\ell$. \rev{Note that under the budget $\mathcal{Q}_M$ with $M \ll N$, the original measure $\mu_1$ is infeasible whenever its support exceeds $M$ points; hence ``no condensation'' is never a solution of \eqref{eq:dc}, and exact equality $\mu_1 = \mu_2$ is not the standard of success (see also Example~\ref{ex:zero_gd}). For generator-parameterized methods (\Cref{sec:space}), the feasible set is instead the push-forward family $(g_d)^{\#} \mathcal{P} (\mathbb{R}^m)$ together with an explicit storage account, counting both the stored latents and the generator parameters.} Intuitively, we aim to find a synthetic distribution $\mu_2$ such that the model trained on $\mu_2$ performs well on the original distribution $\mu_1$.

\begin{remark}[Extension to DC with soft label \citep{qin2024label}] \label{rem:soft}
    We extend the standard dataset condensation framework by allowing $L$ labels to take values in the probability simplex $\Delta^{L-1}$, i.e., each sample is of the form $(\mathbf{x}, \tilde{y})$ with $\tilde{y} \in \Delta^{L-1}$. In this setting, the synthetic distribution $\mu_2$ is defined over $\mathcal{X} \times \Delta^{L-1}$, and the loss is generalized to $\ell(h(\mathbf{x}), \tilde{y})$, e.g., using cross-entropy or squared loss. This formulation enables the condensation process to match not only the input distribution but also the uncertainty structure of labels, effectively capturing the joint distribution $\mu(\mathbf{x}, \tilde{y})$ rather than class-conditional distributions alone. In particular, soft labels provide richer supervision by encoding inter-class similarity and ambiguity, which can be especially beneficial in low-data regimes.
    
    \rev{Figure~\ref{fig:soft_label_kcenter} illustrates this effect in a configuration in which the three views visibly differ: two overlapping Gaussian classes with imbalanced priors ($\pi_0 = 0.85$, $\pi_1 = 0.15$). Selection under the marginal $\mu(x)$ covers geometric extremes only. The $k$-center selection under the joint $\mu(x, y)$ with a metric that also accounts for differences in the soft label $\eta(\mathbf{x}) = \mu(y = 1 | \mathbf{x})$ reallocates part of the budget toward the decision boundary and the high-uncertainty band. Per-class selection under $\mu(x|y)$ with equal per-class budgets oversamples the minority class relative to the marginal view.}
    
    \begin{figure}
        \centering
        \includegraphics[width=\linewidth]{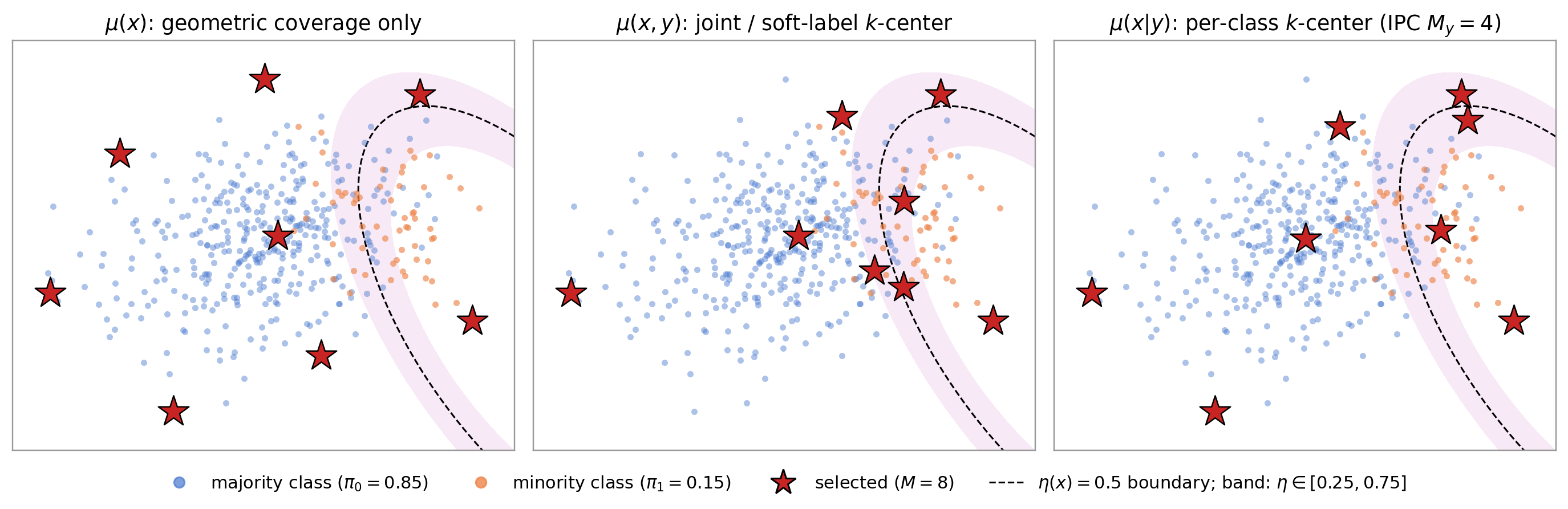}
        \caption{\rev{Comparison of $k$-center selection ($M = 8$) under the marginal $\mu(x)$, the joint $\mu(x, y)$ (soft-label-aware metric), and the class-conditionals $\mu(x|y)$ (equal per-class budgets $M_y = 4$), for two overlapping Gaussian classes with imbalanced priors ($\pi_0 = 0.85$, $\pi_1 = 0.15$). The dashed curve is the Bayes boundary $\eta(\mathbf{x}) = 0.5$; the shaded band marks $\eta(\mathbf{x}) \in [0.25, 0.75]$. The joint view shifts selections toward the boundary and the high-uncertainty region, the marginal view covers geometric extremes, and the per-class view oversamples the minority class.}}
        \label{fig:soft_label_kcenter}
    \end{figure}
\end{remark}

In practice, neither $\mu_1$ or $\mu_2$ is known explicitly. Instead, they are approximated by empirical distributions: $\hat{\mu}_1$ which is induced by the full dataset $\mathcal{T} = \{(\mathbf{x}_i, y_i)\}_{i = 1}^N$, and $\hat{\mu}_2$ which is induced by a small synthetic dataset $\mathcal{S} = \{(\tilde{\mathbf{x}}_j, \tilde{y}_j)\}_{j = 1}^M$ with $M \ll N$. In the context of dataset condensation, the size of condensed dataset $\mathcal{S}$ is expected to be much smaller than the original dataset $\mathcal{T}$, i.e., $M \ll N$. \rev{Crucially, the synthetic object constructed by DC methods \emph{is} the atomic measure $\hat{\mu}_2 = \hat{\mu}_{\mathcal{S}} \in \mathcal{Q}_M$; it is not a sample from an underlying population $\mu_2$. This observation is developed in \Cref{sec:bridge}.} \rev{To better understand the optimization goal, we make the dependence on the learning algorithm explicit. For a deterministic algorithm $\mathcal{A}: \mathcal{P} (\Omega) \to \mathcal{H}$, define the \emph{generalization discrepancy (GD)} of $\mu_2$ relative to $\mu_1$ as
$$\mathrm{GD}_{\mathcal{A}} (\mu_1, \mu_2) := \big|\varphi (\mathcal{A} (\mu_2), \mu_1) - \varphi (\mathcal{A} (\mu_1), \mu_1)\big|,$$
which measures how much generalization performance deteriorates when we train on $\mu_2$ rather than $\mu_1$ \emph{using the given algorithm}. Making the algorithm explicit matters: a definition through unspecified minimizers, $|\varphi (h^*_{\mu_2}, \mu_1) - \varphi (h^*_{\mu_1}, \mu_1)|$, leaves the optimizer, initialization, stopping rule, and most consequentially, the selection among non-unique minimizers undetermined; Example~\ref{ex:zero_gd} below shows that this selection can change the value of the discrepancy. For stochastic training with randomness $\xi$ (initialization, mini-batch order), we take
$$\overline{\mathrm{GD}}_{\mathcal{A}} (\mu_1, \mu_2) := \big|\mathbb{E}_{\xi} [\varphi (\mathcal{A} (\mu_2; \xi), \mu_1)] - \mathbb{E}_{\xi} [\varphi (\mathcal{A} (\mu_1; \xi), \mu_1)]\big|$$
as the primary definition, and note the stronger seed-coupled variant $\mathbb{E}_{\xi} \big|\varphi (\mathcal{A} (\mu_2; \xi), \mu_1) - \varphi (\mathcal{A} (\mu_1; \xi), \mu_1)\big|$, which dominates it by Jensen's inequality; the bound below controls both. Note that minimizing $\varphi (\mathcal{A} (\mu_2), \mu_1)$ over $\mu_2$ is equivalent to minimizing $\mathrm{GD}_{\mathcal{A}} (\mu_1, \mu_2)$, since the reference term does not depend on $\mu_2$. The central quantitative relationship between this task-level discrepancy and distribution-level discrepancies is the following.}

\begin{proposition}[\rev{Transfer bound via the loss-induced class}] \label{prop:transfer}
\rev{Let $\ell$ be a loss, write $R_{\mu} (h) := \varphi (h, \mu)$, and define the \emph{loss-induced test class} $\mathcal{F}_{\ell} := \{z = (\mathbf{x}, y) \mapsto \ell (h (\mathbf{x}), y): h \in \mathcal{H}\}$. Let $\mathcal{F} \supseteq \mathcal{F}_{\ell}$, and suppose the algorithm $\mathcal{A}$ returns $\varepsilon$-approximate risk minimizers over $\mathcal{H}$, i.e., $R_{\nu} (\mathcal{A} (\nu)) \le \inf_{h \in \mathcal{H}} R_{\nu} (h) + \varepsilon$ for every input measure $\nu$. Then, for all $\mu_1, \mu_2$,
$$\mathrm{GD}_{\mathcal{A}} (\mu_1, \mu_2) \le 2 \, \mathrm{IPM} (\mu_1, \mu_2; \mathcal{F}) + 2 \varepsilon.$$
The same bound holds for $\overline{\mathrm{GD}}_{\mathcal{A}}$ and for the seed-coupled variant whenever $\mathcal{A} (\cdot; \xi)$ is $\varepsilon (\xi)$-approximate with $\mathbb{E}_{\xi} [\varepsilon (\xi)] \le \varepsilon$.}
\end{proposition}
\begin{proof}
\rev{Write $h_i = \mathcal{A} (\mu_i)$. Since $\mathcal{F}_{\ell} \subseteq \mathcal{F}$, we have $|R_{\mu_1} (h) - R_{\mu_2} (h)| \le \mathrm{IPM} (\mu_1, \mu_2; \mathcal{F})$ for every $h \in \mathcal{H}$. Then
$$R_{\mu_1} (h_2) - R_{\mu_1} (h_1) = \underbrace{[R_{\mu_1} (h_2) - R_{\mu_2} (h_2)]}_{\le \, \mathrm{IPM}} + \underbrace{[R_{\mu_2} (h_2) - R_{\mu_2} (h_1)]}_{\le \, \varepsilon} + \underbrace{[R_{\mu_2} (h_1) - R_{\mu_1} (h_1)]}_{\le \, \mathrm{IPM}},$$
where the middle term uses the $\varepsilon$-approximate optimality of $h_2$ for $R_{\mu_2}$ together with $h_1 \in \mathcal{H}$. Conversely, $R_{\mu_1} (h_1) - R_{\mu_1} (h_2) \le R_{\mu_1} (h_1) - \inf_{h \in \mathcal{H}} R_{\mu_1} (h) \le \varepsilon$. Combining both directions gives the (slightly stronger) bound $\mathrm{GD}_{\mathcal{A}} \le 2 \, \mathrm{IPM} (\mu_1, \mu_2; \mathcal{F}) + \varepsilon$, which implies the stated one. The stochastic statements follow by applying the argument for each realization of $\xi$ and taking expectations (Jensen's inequality for $\overline{\mathrm{GD}}_{\mathcal{A}}$). Bounds of this family are classical in domain adaptation \citep{bendavid2010theory, mansour2009domain}.}
\end{proof}

\rev{Proposition~\ref{prop:transfer} is the theoretical center of the paper.} This bound reveals two key insights:
\begin{itemize}
    \item \rev{Minimizing a distribution discrepancy over a class $\mathcal{F} \supseteq \mathcal{F}_{\ell}$ is a principled surrogate objective, for two reasons. It is computationally cheaper: it is a single-level problem, whereas evaluating $\mathrm{GD}_{\mathcal{A}}$ requires solving the inner training problem at every step. And it is \emph{robust}: the bound controls the transfer gap uniformly over minimizer selections and over any algorithm returning $\varepsilon$-approximate minimizers, whereas a small value of $\mathrm{GD}_{\mathcal{A}}$ itself can rely on a favorable tie-break among minimizers (Example~\ref{ex:zero_gd}).}

    \item \rev{Distribution matching is sufficient but not necessary: mutually singular measures can induce identical task performance (Example~\ref{ex:zero_gd}), so a large distribution discrepancy does not preclude successful condensation. In particular, a small transfer gap achieved by a distribution very different from $\mu_1$ is a success of condensation, not a failure of the criterion.}
\end{itemize}
Indeed, generalization discrepancy requires solving a bi-level optimization problem, i.e., training a model at each iteration to evaluate the outer objective, whereas minimizing the distribution discrepancy typically involves a single-level optimization over distributional parameters or dataset embeddings.

\begin{example}[Zero generalization discrepancy under distribution mismatch] \label{ex:zero_gd}
Let $\mathcal{X} = \mathbb{R}$ and $\mathcal{Y} = \{-1, +1\}$. Consider the hypothesis class $\mathcal{H} := \{ h_t(x) = \mathrm{sign}(x - t) : t \in \mathbb{R} \}$ and the $0$--$1$ loss $\ell(h(x), y) = \mathbf{1}\{h(x) \neq y\}$. Define two distributions $\mu_1$ and $\mu_2$ on $\mathcal{X} \times \mathcal{Y}$ as follows:
\begin{itemize}
    \item $\mu_1$: $x \sim \mathrm{Uniform}([-1,1])$, and $y = \mathrm{sign}(x)$;
    \item $\mu_2$: $\mu_2 = \tfrac{1}{2}\delta_{(-1,-1)} + \tfrac{1}{2}\delta_{(1,1)}$.
\end{itemize}

For $\mu_1$, the Bayes-optimal classifier is $h^*_{\mu_1}(x) = \mathrm{sign}(x)$, which achieves zero risk $\varphi(h^*_{\mu_1}, \mu_1) = 0$. For $\mu_2$, any threshold $t \in (-1,1)$ yields zero classification error on the support $\{-1,1\}$. Hence, $h^*_{\mu_2} \in \{ \mathrm{sign}(x - t) : t \in (-1,1) \}$, and in particular, $h(x) = \mathrm{sign}(x)$ is one such minimizer. Therefore, $\varphi(h^*_{\mu_2}, \mu_2) = 0$. Selecting $h^*_{\mu_2}(x) = \mathrm{sign}(x)$, we evaluate its performance on $\mu_1$: $\varphi(h^*_{\mu_2}, \mu_1) = 0 = \varphi(h^*_{\mu_1}, \mu_1)$. Thus, the generalization discrepancy satisfies $\mathrm{GD}(\mu_1, \mu_2; \mathcal{H}) = 0$. Although $\mu_1 \neq \mu_2$ (they have different supports and densities), they even induce a common optimal classifier within $\mathcal{H}$ for this specific setting. Consequently, training on $\mu_2$ yields the same generalization performance on $\mu_1$ as training on $\mu_1$ itself.

\rev{Two lessons follow. (i) \emph{Selection sensitivity of GD}: the value $0$ is attained only under a favorable tie-break among the minimizers for $\mu_2$; any other minimizer $h_t$ with $t \in (-1, 1) \setminus \{0\}$ incurs risk $\varphi (h_t, \mu_1) = |t| / 2 > 0$ on $\mu_1$. The algorithm-indexed definition $\mathrm{GD}_{\mathcal{A}}$ resolves this ambiguity, and the bound of \Cref{prop:transfer} holds uniformly over such selections. (ii) \emph{Non-necessity of distribution matching}: $\mu_2$ is supported on two atoms and is mutually singular with $\mu_1$, yet (under the favorable selection) it reproduces the generalization performance of a continuous distribution with a two-point synthetic set, which is \emph{successful} condensation, demonstrating that a small distribution discrepancy is sufficient (\Cref{prop:transfer}) but not necessary for a small transfer gap.}
\end{example}

Let us now take a closer look at the generalization discrepancy \rev{$\mathrm{GD}_{\mathcal{A}} (\mu_1, \mu_2)$}. Recall that this quantity measures how much the test performance (on distribution $\mu_1$) degrades when we train a model on a different distribution $\mu_2$. In general, $\varphi (h, \mu) = \mathbb{E}_{(\mathbf{x}, y) \sim \mu} [\ell (h (\mathbf{x}), y)]$ is unbounded due to the potential unboundedness of the loss function $\ell$ and the hypothesis $h$. However, this behavior can be controlled by restricting the hypotheses to a uniformly bounded class over a compact input domain. For instance, if we assume normalized inputs $\mathbf{x} \in [0,1]^n$ and continuous hypotheses $h_{\theta}: [0,1]^n \rightarrow \mathbb{R}$ with uniformly bounded parameters $\|\theta\| \le B$, then the image of any model $h \in \mathcal{H}$, that is, $h ([0, 1]^n])$, is contained in a compact set $\mathbf{K} \subseteq \mathbb{R}$. Assuming further that the loss function $\ell (\cdot, \cdot)$ is $L$-Lipschitz continuous in its first argument over $\mathbf{K}$, we obtain 
$$\rev{\mathrm{GD}_{\mathcal{A}} (\mu_1, \mu_2)} \le L \cdot \|h^*_{\mu_1} - h^*_{\mu_2}\|, \quad \rev{h^*_{\mu_i} := \mathcal{A} (\mu_i),}$$
where $\|h^*_{\mu_1} - h^*_{\mu_2}\| := \sup_{\mathbf{x} \in [0, 1]^n} |h^*_{\mu_1} (\mathbf{x}) - h^*_{\mu_2} (\mathbf{x})|$ denotes the uniform (sup) norm over inputs. Now suppose both models $h^*_{\mu_1}$ and $h^*_{\mu_2}$ belong to a neural network class with fixed architecture, and let $\theta^*_{\mu_1}$ and $\theta^*_{\mu_2}$ be the parameter vectors of these trained models. If the neural network is Lipschitz continuous with respect to its parameters, i.e., $\sup_{\red{\mathbf{x}} \in [0,1]^n} |h^* (\theta^*_{\mu_1}, \red{\mathbf{x}}) - h^* (\theta^*_{\mu_2}, \red{\mathbf{x}})| \le C \cdot \|\theta^*_{\mu_1} - \theta^*_{\mu_2}\|$, for some global constant $C$. Then we obtain: 
$$\|h^*_{\mu_1} - h^*_{\mu_2}\| \le C \cdot \|\theta^*_{\mu_1} - \theta^*_{\mu_2}\|.$$
This motivates two additional discrepancy measures\rev{, where $h^*_{\mu} := \mathcal{A} (\mu)$ and $\theta^*_{\mu}$ denote the model and parameters returned by the (possibly seed-coupled) training runs}:
\begin{enumerate}
    \item \textbf{Value discrepancy (VD)}, defined by $\mathrm{VD} (\mu_1, \mu_2; \mathcal{H}) := \|h^*_{\mu_1} - h^*_{\mu_2}\|$;

    \item \textbf{Parameter discrepancy (PD)}\rev{. Parameter distance is not intrinsic: common network parameterizations admit neuron permutations and, for positively homogeneous activations, rescaling symmetries, so two parameter vectors can be far apart while representing the same function \citep{ainsworth2023git}. We therefore define PD as a \emph{symmetry-aligned} parameter discrepancy over the symmetry group $\mathcal{G}$ of the architecture (permutations, and rescalings where applicable):
    $$\mathrm{PD}_{\mathcal{G}} (\mu_1, \mu_2; \mathcal{H}) := \inf_{\pi \in \mathcal{G}} \|\theta^*_{\mu_1} - \pi (\theta^*_{\mu_2})\|.$$ When $\mathcal{G}$ is a finite group acting by isometries (permutations only), the infimum is attained and $\mathrm{PD}_{\mathcal{G}}$ is a genuine metric on the quotient space. Positive rescalings, however, form a non-compact group that does not act isometrically under the Euclidean norm; in that case the infimum need not be attained, and $\mathrm{PD}_{\mathcal{G}}$ need not be symmetric or satisfy the triangle inequality. We therefore claim no metric properties in general: the upper-bound argument below uses only the alignment infimum, which is all that is required.}
\end{enumerate}
\rev{Under the following assumptions: (a) a fixed architecture and parameterization shared by both training runs; (b) the parameter-Lipschitz condition above holding uniformly over the $\mathcal{G}$-orbit, i.e., $\sup_{\mathbf{x} \in [0, 1]^n} |h (\theta, \mathbf{x}) - h (\theta', \mathbf{x})| \le C \inf_{\pi \in \mathcal{G}} \|\theta - \pi (\theta')\|$; and (c) a coupling of the two training runs (shared initialization and randomness), so that the returned parameter vectors are comparable}, these relationships imply a hierarchy of upper bounds on generalization discrepancy:
$$\rev{\mathrm{GD}_{\mathcal{A}} (\mu_1, \mu_2)} \le L \cdot \mathrm{VD} (\mu_1, \mu_2; \mathcal{H}) \le LC \cdot \rev{\mathrm{PD}_{\mathcal{G}}} (\mu_1, \mu_2; \mathcal{H}).$$
Thus, value discrepancy and parameter discrepancy provide more tractable surrogates for bounding or approximating generalization performance loss in practice\rev{, \emph{when the stated assumptions hold}}. They also highlight that preserving either model predictions or parameter configurations under distribution shifts can serve as effective criteria for dataset condensation. \rev{Assumption (c) is not vacuous, and it reconciles the hierarchy with the later discussion of trajectory matching: trajectory matching \citep{cazenavette2022dataset} implements precisely such a coupling by initializing the synthetic-data run at checkpoints of the real-data trajectory (\Cref{sec:metric}), which is what makes its parameter-space comparison meaningful; without a coupling and the symmetry alignment, a small raw parameter distance is neither necessary nor sufficient for functional closeness.}

Another classical method to verify whether two probability distributions are identical is through their characteristic functions. For a distribution $\mu$, the characteristic function is defined as: $F_{\mu} (\mathbf{t}) := \mathbb{E}_{\mathbf{x} \sim \mu} [e^{i \langle \mathbf{x}, \mathbf{t}\rangle}]$ for $\mathbf{t} \in \mathbb{R}^n$. Two distributions $\mu_1$ and $\mu_2$ are equal if and only if $F_{\mu_1} (\mathbf{t}) = F_{\mu_2} (\mathbf{t})$ for all $\mathbf{t}$. This motivates the definition of \emph{characteristic discrepancy (CD)} defined as $\mathrm{CD} (\mu_1, \mu_2; \mathcal{H}) := \|F_{\mu_1} - F_{\mu_2}\|$, where the norm is typically taken to be the $L_{\infty}$ or $L_2$ norm over $\mathbf{t} \in \mathbb{R}^n$. \rev{The comparison with IPMs requires care about the choice of norm; two precise statements hold. For the $L_{\infty}$ version, $\mathrm{CD}_{\infty} (\mu_1, \mu_2) \le 2 \, \mathrm{IPM} (\mu_1, \mu_2; \mathcal{F})$ whenever the real and imaginary parts $\{\cos \langle \mathbf{t}, \cdot \rangle, \sin \langle \mathbf{t}, \cdot \rangle: \mathbf{t} \in \mathbb{R}^n\}$ belong to the test class $\mathcal{F}$, which is a genuine function-class-inclusion argument. The $L_2$ version, in contrast, requires a specified frequency measure $\rho$ on $\mathbb{R}^n$, $\mathrm{CD}_{L_2 (\rho)} (\mu_1, \mu_2) := \|F_{\mu_1} - F_{\mu_2}\|_{L_2 (\rho)}$, and by Bochner's theorem it is \emph{exactly} the MMD associated with the translation-invariant kernel $k_{\rho} (\mathbf{x}_1, \mathbf{x}_2) = \int e^{i \langle \mathbf{t}, \mathbf{x}_1 - \mathbf{x}_2 \rangle} \, d\rho (\mathbf{t})$ whose spectral measure is $\rho$ \citep{sriperumbudur2010hilbert}; this identification does not follow from function-class inclusion. It also gives the neural characteristic function method of \Cref{sec:metric} a principled reading as MMD with a \emph{learned} spectral measure.} Importantly, the characteristic discrepancy is purely distributional: it measures the closeness of two distributions in a general sense, without reference to a learning task. Therefore, it provides no direct guarantee on the generalization performance of models trained on one distribution and tested on another. 

To summarize the relationships among the discussed discrepancy measures, we provide the following schematic diagram. Arrows indicate an upper bound (up to a constant), i.e., $A \to B$ implies $A \le c B$ for some constant $c < \infty$. \rev{Each arrow holds under the assumptions stated above: the $\mathrm{GD} \to \mathrm{DD}$ arrow requires a test class $\mathcal{F} \supseteq \mathcal{F}_{\ell}$ and an $\varepsilon$-approximate learner (\Cref{prop:transfer}; the constant absorbs the $2 \varepsilon$ term), the $\mathrm{GD} \to \mathrm{VD}$ arrow requires the bounded-Lipschitz setting above, and the $\mathrm{VD} \to \mathrm{PD}$ arrow additionally requires the symmetry-aligned discrepancy $\mathrm{PD}_{\mathcal{G}}$ and coupled training runs.}
$$
\begin{tikzpicture}
    \tikzset{dummy/.style={draw, rectangle, fill=blue!10, thick, inner sep=4pt, rounded corners=1pt}}
    \tikzset{neuron/.style={draw, circle, inner sep=0, outer sep=0, minimum size=1.1cm}};
    \tikzset{box/.style={draw, rectangle, inner sep=0, minimum width=4cm, minimum height=1cm}};

    \node[dummy] (cd) at (0,0) {$\mathrm{CD} (\mu_1, \mu_2; \mathcal{H})$};
    \node[dummy] (d) at (2,1) {$\mathrm{DD} (\mu_1, \mu_2; \mathcal{H})$};
    \node[dummy] (gd) at (4,0) {$\mathrm{GD} (\mu_1, \mu_2; \mathcal{H})$};
    \node[dummy] (vd) at (7,.5) {$\mathrm{VD} (\mu_1, \mu_2; \mathcal{H})$};
    \node[dummy] (pd) at (10,1) {$\rev{\mathrm{PD}_{\mathcal{G}}} (\mu_1, \mu_2; \mathcal{H})$};

    \draw [->,thick] (cd) -- (d); 
    \draw [->,thick] (gd) -- (d);
    \draw [->,thick] (gd) -- (vd);
    \draw [->,thick] (vd) -- (pd);
\end{tikzpicture}
$$
    
In the context of DC methods, these discrepancies serve as tools to quantify the mismatch between the source and target distributions. By choosing a suitable discrepancy with balance between computational tractability and task relevance, one can guide adaptation strategies that reduce distributional shifts, thereby improving cross-domain generalization.

\rev{\subsection{From Empirical Objectives to Population Guarantees} \label{sec:bridge}
The theory above is stated at the population level, whereas the methods surveyed in \Cref{sec:metric} optimize the empirical objective $D (\hat{\mu}_{\mathcal{T}}, \hat{\mu}_{\mathcal{S}}; \mathcal{H})$. One structural observation makes the bridge between the two simpler than in standard two-sample settings: in DC, the synthetic measure is not a sample from an underlying population, it \emph{is} the atomic measure $\hat{\mu}_{\mathcal{S}} \in \mathcal{Q}_M$, and the constraint $M \ll N$ enters through the feasible set of \eqref{eq:dc}, not through sampling. The only estimation error is therefore on the real-data side.

\begin{proposition}[Finite-sample bridge] \label{prop:bridge}
Let $\mathcal{F}$ be a class of functions uniformly bounded by $B$, let $\mathcal{T} = \{z_i\}_{i = 1}^N$ be drawn i.i.d.\ from $\mu_1$, and let $\hat{\mu}_{\mathcal{S}}$ be \emph{any} measure supported on at most $M$ atoms, possibly depending on $\mathcal{T}$. Then, with probability at least $1 - \delta$ over the draw of $\mathcal{T}$,
$$\mathrm{IPM} (\mu_1, \hat{\mu}_{\mathcal{S}}; \mathcal{F}) \le \mathrm{IPM} (\hat{\mu}_{\mathcal{T}}, \hat{\mu}_{\mathcal{S}}; \mathcal{F}) + 2 \mathfrak{R}_N (\mathcal{F}) + B \sqrt{\frac{2 \log (2 / \delta)}{N}},$$
where $\mathfrak{R}_N (\mathcal{F})$ is the Rademacher complexity of $\mathcal{F}$ over samples of size $N$. The analogous statement holds for MMD with a kernel bounded by $K$, with the last two terms replaced by $2 \sqrt{K / N} + \sqrt{2 K \log (2 / \delta) / N}$ \citep{gretton2012kernel}.
\end{proposition}
\begin{proof}
By the triangle inequality for IPMs,
$$\mathrm{IPM} (\mu_1, \hat{\mu}_{\mathcal{S}}; \mathcal{F}) \le \mathrm{IPM} (\hat{\mu}_{\mathcal{T}}, \hat{\mu}_{\mathcal{S}}; \mathcal{F}) + \mathrm{IPM} (\mu_1, \hat{\mu}_{\mathcal{T}}; \mathcal{F}).$$
The second term, $\sup_{f \in \mathcal{F}} |\mathbb{E}_{\mu_1} f - \frac{1}{N} \sum_i f (z_i)|$, changes by at most $2 B / N$ when a single sample is replaced, so McDiarmid's bounded-difference inequality together with the standard symmetrization bound on its expectation \citep{bartlett2002rademacher, boucheron2013concentration} yields the stated deviation bound (the factor $2 / \delta$ accounts for the two-sided supremum). Since $\hat{\mu}_{\mathcal{S}}$ enters only through the triangle inequality, the bound holds \emph{simultaneously for all} synthetic measures, and no term of order $M^{-1/2}$ arises, which correctly reflecting that $\hat{\mu}_{\mathcal{S}}$ is not sampled.
\end{proof}

Combining \Cref{prop:transfer,prop:bridge}: if $\mathcal{F} \supseteq \mathcal{F}_{\ell}$ is uniformly bounded by $B$ and $\mathcal{A}$ returns $\varepsilon$-approximate minimizers, then with probability at least $1 - \delta$,
$$\mathrm{GD}_{\mathcal{A}} (\mu_1, \hat{\mu}_{\mathcal{S}}) \le 2 \, \mathrm{IPM} (\hat{\mu}_{\mathcal{T}}, \hat{\mu}_{\mathcal{S}}; \mathcal{F}) + 4 \mathfrak{R}_N (\mathcal{F}) + 2 B \sqrt{\frac{2 \log (2 / \delta)}{N}} + 2 \varepsilon,$$
i.e., the empirical objective that DC methods actually optimize controls the population transfer gap, up to optimization error and an estimation term of order $\mathfrak{R}_N (\mathcal{F}) + \sqrt{\log (1 / \delta) / N}$. Uniform boundedness of $\mathcal{F}$ alone does not make $\mathfrak{R}_N (\mathcal{F})$ vanish, because a sufficiently rich bounded class has non-vanishing Rademacher complexity. Thus, an $O_P (N^{-1/2})$ rate requires an additional complexity assumption on $\mathcal{F}$, such as finite cardinality, bounded VC or pseudo-dimension, or a suitable entropy condition \citep{bartlett2002rademacher}; for MMD with a bounded kernel the rate holds without further assumptions. Statements that we make only at the population level (e.g., the injectivity results of \Cref{prop:discrepancy}) are explicitly scoped as such and do not transfer to the $M \ll N$ regime without this bridge.

\paragraph{Three sources of error.} It is useful to keep separate: (i) \emph{estimation error}, between $\mu_1$ and $\hat{\mu}_{\mathcal{T}}$, of order $\mathfrak{R}_N (\mathcal{F}) + \sqrt{\log (1 / \delta) / N}$ (\Cref{prop:bridge}; $O_P (N^{-1/2})$ only under a complexity assumption on $\mathcal{F}$); (ii) \emph{optimization error}, the suboptimality $\varepsilon$ of the learner together with the suboptimality of the outer minimization over $\mathcal{S}$, which our bounds treat as given quantities; and (iii) \emph{transfer error}, the passage from a small discrepancy to downstream performance, which \Cref{prop:transfer} quantifies and which hinges on the inclusion $\mathcal{F} \supseteq \mathcal{F}_{\ell}$. Empirical comparisons of DC methods conflate the three; the decomposition explains, for instance, why a method with a smaller measured surrogate discrepancy need not achieve better downstream accuracy if its test class does not dominate $\mathcal{F}_{\ell}$.}

\section{Change of Space} \label{sec:main}
In the previous section, we discussed various types of discrepancies that measure the similarity between two probability distributions. Some of these are proper distance metrics, while others not only quantify distributional similarity but also provide guarantees on generalization performance. 
At a high level, dataset condensation aims to minimize the discrepancy between a target distribution $\mu_1$ (e.g., the original data distribution) and a synthesized distribution $\mu_2$, by solving: $\min_{\mu_2} D (\mu_1, \mu_2; \mathcal{H})$. This formulation implicitly assumes that both $\mu_1$ and $\mu_2$ are defined over a common input space $\Omega \subseteq \mathbb{R}^n$. Now consider a more general case where there is a measurable mapping $g: (\mathbb{R}^n, \mathcal{F}_1) \to (\mathbb{R}^m, \mathcal{F}_2)$ between two measurable spaces, and where the original space $\mathbb{R}^n$  is endowed with a measure $\mu: \mathcal{F}_1 \to [0, +\infty]$. The \emph{push-forward measure} of $\mu$ under $g$, denoted by $g^{\#}\mu$, is a measure on $\mathbb{R}^m$ defined as:
$$g^{\#}\mu (U) = \mu \circ g^{-1} (U), \; \forall \; U \in \mathcal{F}_2.$$

In deep learning applications, high-dimensional input spaces often lead to optimization challenges, both computationally and statistically, due to the curse of dimensionality. One common strategy to mitigate this issue is to employ latent variable generative models that transform high-dimensional data into a lower-dimensional feature space. A prominent example is the autoencoder, which consists of two components: an \emph{encoder} $g_e: \mathbb{R}^n \to \mathbb{R}^m$ which maps high-dimensional inputs to a low-dimensional latent space, and a \emph{decoder} $g_d: \mathbb{R}^m \to \mathbb{R}^n$ which reconstructs the input from its latent representation. The decoder $g_d$ is intended to approximate the inverse of the encoder $g_e$, so that their composition $g_d \circ g_e$ approximates the identity function $\text{id}_{\mathbb{R}^n}$, or at least locally approximates $\text{id}_{\Omega}$ for a compact domain $\Omega \subseteq \mathbb{R}^n$. 

This structure enables us to push forward probability measures from the original space to the latent space via the encoder. That is, given a measure $\mu$ over $\mathbb{R}^n$, we define the push-forward measure $g_e^{\#}\mu := \mu \circ g_e^{-1}$, which is a measure on $\mathbb{R}^m$. With this formulation, we gain additional flexibility in choosing the space where we perform distribution matching and optimization, either in the original input space or in the latent feature space. See \Cref{tab:change_space} for a summary of possible choices. Importantly, if optimization is conducted in the latent space, we must use the push-forward (via the decoder) to recover the corresponding distribution in the input space.
\begin{table}[th]
\caption{Regimes of optimizing and matching distributions in input and latent spaces.}
    \label{tab:change_space}
    \centering
    \begin{tabular}{|c|c|c|}
        \hline
        \diagbox{Match}{Optimize}& Input space $\mathbb{R}^n$ & Latent space $\mathbb{R}^m$ \\
        \hline
        Input space $\mathbb{R}^n$ & $\displaystyle\min_{\red{\mu_2} \in \mathcal{P} (\mathbb{R}^n)} D (\mu_1, \red{\mu_2}; \mathcal{H})$ & $\displaystyle\min_{\red{\nu} \in \mathcal{P} (\mathbb{R}^m)} D (\mu, g_d^{\#}\red{\nu}; \mathcal{H})$ \\
        \hline
        Latent space $\mathbb{R}^m$ & $\displaystyle\min_{\red{\mu_2} \in \mathcal{P} (\mathbb{R}^n)} D (g_e^{\#}\mu_1, g_e^{\#}\red{\mu_2}; \mathcal{H})$ & $\displaystyle\min_{\red{\nu} \in \mathcal{P} (\mathbb{R}^m)} D (g_e^{\#}\mu, \red{\nu}; \mathcal{H})$ \\
        \hline
    \end{tabular}
\end{table}

According to \Cref{prop:discrepancy}, for suitable choices of function class $\mathcal{H}$, the discrepancy measure $D (\mu_1, \mu_2; \mathcal{H})$ becomes a valid distance metric over distributions on the input space $\mathbb{R}^n$. A natural question arises: {\em Does this metric property still hold when we change the space in which distribution matching and optimization take place?} \rev{All statements in this section are organized around a single change-of-variables identity. For a measurable map $g: \mathbb{R}^n \to \mathbb{R}^m$, a test class $\mathcal{F}$ of integrable functions on $\mathbb{R}^m$, and measures $\mu, \nu$ on $\mathbb{R}^n$,
\begin{align} \label{eq:pullback}
    D (g^{\#} \mu, g^{\#} \nu; \mathcal{F}) = D (\mu, \nu; \mathcal{F} \circ g), \qquad \mathcal{F} \circ g := \{f \circ g: f \in \mathcal{F}\},
\end{align}
which holds for IPM-type discrepancies since $\int f \, d (g^{\#} \mu) = \int f \circ g \, d\mu$. Matching push-forward measures therefore compares the original measures through the \emph{pullback test class} $\mathcal{F} \circ g$ only; information is preserved exactly when this pullback class is measure-determining for the family of measures under consideration.} We consider three distinct cases involving push-forward measures:
\begin{enumerate}
    \item \textbf{Push-forward both distributions via the encoder}: Define 
    $$d (\mu_1, \mu_2) := D (g_e^{\#}\mu_1, g_e^{\#}\mu_2; \mathcal{H}).$$
    \rev{By \eqref{eq:pullback}, $d (\mu_1, \mu_2) = D (\mu_1, \mu_2; \mathcal{H} \circ g_e)$, so $d$ separates $\mu_1$ from $\mu_2$ if and only if the pullback class $\mathcal{H} \circ g_e$ separates them. If $\mathcal{H}$ is measure-determining on $\mathbb{R}^m$, a sufficient condition is that $g_e$ is Borel and \emph{injective on $\mathrm{supp} (\mu_1) \cup \mathrm{supp} (\mu_2)$}: by the Lusin--Suslin theorem, an injective Borel map on a Borel subset of a standard Borel space has a Borel image and a Borel inverse on that image \citep[Ch.~15]{kechris1995classical}, so $g_e^{\#} \mu_1 = g_e^{\#} \mu_2$ forces $\mu_1 = \mu_2$. Global injectivity of $g_e$ on all of $\mathbb{R}^n$ is sufficient but \emph{not} necessary for a fixed family of data distributions.

    \begin{example}[Information loss under a non-injective encoder] \label{ex:noninj}
    Let $g_e (\mathbf{x}) = \|\mathbf{x}\|_2$ on $\mathbb{R}^2$, and let $\mu_1$ and $\mu_2$ be uniform on two \emph{distinct} pairs of antipodal points of the unit circle. Then $g_e^{\#} \mu_1 = g_e^{\#} \mu_2 = \delta_1$, so every latent-space discrepancy vanishes while $\mu_1 \neq \mu_2$: the pullback class $\mathcal{H} \circ g_e$ contains only radial functions and is not measure-determining for this pair. Injectivity on $\mathrm{supp} (\mu_1) \cup \mathrm{supp} (\mu_2)$ fails, exactly as \eqref{eq:pullback} predicts.
    \end{example}}
    
    \item \textbf{Push-forward only the synthetic distribution via the decoder}: Define
    $$d(\mu, \nu) = D (\mu, g_d^{\#} \nu; \mathcal{H}),$$
    where $\mu$ is \rev{defined} over $\mathbb{R}^n$ and $\nu$ over $\mathbb{R}^m$. \rev{The relevant quantity for decoder-based synthesis is the approximation error
    $$\inf_{\tau \in \mathcal{P} (\mathbb{R}^m)} D (\mu, g_d^{\#} \tau; \mathcal{H}),$$
    which measures how well the decoder family $\{g_d^{\#} \tau: \tau \in \mathcal{P} (\mathbb{R}^m)\}$ can represent $\mu$. Two points deserve emphasis. First, exact representability, $\mu = g_d^{\#} \tau$ for some $\tau$, is a statement about the \emph{measure-level range} of $g_d$ for the particular $\mu$ at hand; surjectivity of $g_d$ as a map is neither necessary for a given $\mu$ nor the right condition to state pointwise. Second, absent a closedness assumption on the family $\{g_d^{\#} \tau\}$ with respect to the topology induced by $D$, a \emph{zero infimum} implies only that $\mu$ lies in the closure of this family, and the infimum need not be attained.}

    \item \textbf{Push-forward only the target distribution via the encoder}: Define
    $$d(\mu, \nu) = D (g_e^{\#}\mu, \nu; \mathcal{H}),$$
    where again $\mu$ is over $\mathbb{R}^n$ and $\nu$ over $\mathbb{R}^m$. \rev{Assuming only that $D (\rho, \rho; \mathcal{H}) = 0$ for all $\rho$, the choice $\nu = g_e^{\#} \mu$ always attains $d (\mu, \nu) = 0$; \emph{no surjectivity of the encoder is required}. If the synthetic distribution over $\mathbb{R}^n$ is then defined as $g_d^{\#} \nu = (g_d \circ g_e)^{\#} \mu$, recovery of $\mu$ amounts to the \emph{distribution-specific} condition
    $$(g_d \circ g_e)^{\#} \mu = \mu,$$
    which is strictly weaker than the pointwise identity $g_d \circ g_e = \mathrm{id}_{\mathbb{R}^n}$ on the entire input space: it is implied by $g_d \circ g_e = \mathrm{id}$ holding $\mu$-almost everywhere, and it can hold even when $g_d \circ g_e$ merely permutes mass within $\mathrm{supp} (\mu)$ (e.g., $\mu$ uniform on $\{\pm \mathbf{e}_1\}$ and $g_d \circ g_e (\mathbf{x}) = -\mathbf{x}$).}
\end{enumerate}
From the above discussion, we see that there is no free lunch in dataset condensation: while performing optimization or distribution matching directly in the original input space ensures full fidelity, it is computationally expensive due to high dimensionality. On the other hand, dimensionality reduction via an autoencoder offers computational benefits, but \rev{incurs an approximation error with two distinct sources: the pullback test class $\mathcal{H} \circ g_e$ may fail to be measure-determining for the data distributions (information loss in \emph{matching}, cf.\ Example~\ref{ex:noninj}), and the decoder family $\{g_d^{\#} \tau\}$ may fail to contain, or even to have in its closure, the target measure (information loss in \emph{synthesis})}. We summarize the relationships among the different matching strategies and their required properties in the following diagram:
$$
\begin{tikzpicture}
    \tikzset{dummy/.style= {inner sep=0, outer sep=0}}
    \tikzset{neuron/.style={draw, circle, inner sep=0, outer sep=0, minimum size=1.1cm}};
    \tikzset{box/.style={draw, rectangle, inner sep=0, minimum width=4cm, minimum height=1cm}};

    \node[dummy] (input) at (0,3) {$\mathcal{P} (\mathbb{R}^n)$};
    \node[dummy] (latent) at (8,3) {$\mathcal{P} (\mathbb{R}^m)$};
    
    \node[neuron] (S) at (0,0) {$\mu_1$};
    \node[neuron] (psiT) at (0,2) {$g_d^{\#}\red{\nu}$};
    \node[neuron] (S') at (0,-2) {$\red{\mu_2}$};

    \node[dummy] (surj) at (-1,1) {\rev{approx.\ error}};
    \node[dummy] (free) at (-1,-1) {free};

    \node[box] (psi) at (4,2) {Decoder $g_d: \mathbb{R}^m \to \mathbb{R}^n$};
    \node[box] (phi) at (4,-1) {Encoder $g_e: \mathbb{R}^n \to \mathbb{R}^m$};

    \node[neuron] (phiS) at (8,0) {$g_e^{\#}\mu_1$};
    \node[neuron] (T) at (8,2) {$\red{\nu}$};
    \node[neuron] (phiS') at (8,-2) {$g_e^{\#}\red{\mu_2}$};

    \node[dummy] (id) at (9.5,1) {\rev{$(g_d \circ g_e)^{\#} \mu_1 = \mu_1$}};
    \node[dummy] (inj) at (9,-1) {\rev{$g_e$ inj.\ on supp.}};

    \draw [->] (T) -- (psi); \draw [->] (psi) -- (psiT);
    \draw [->] (S) -- (phi.north west); \draw [->] (phi.north east) -- (phiS);
    \draw [->] (S') -- (phi.south west); \draw [->] (phi.south east) -- (phiS');

    \draw [<->] (psiT) -- (S); \draw [<->] (S) -- (S');
    \draw [<->] (T) -- (phiS); \draw [<->] (phiS) -- (phiS');
\end{tikzpicture}
$$

In practice, when given a distribution $\mu$ over $\mathbb{R}^n$, the expectation under its push-forward $g^{\#} \mu$ over $\mathbb{R}^m$ can be computed using the change of variable formula $\mathbb{E}_{\mathbf{z} \sim g^{\#} \mu} [\mathbf{z}] = \mathbb{E}_{\mathbf{x} \sim \mu} [g (\mathbf{x})]$. Given a finite dataset $\mathcal{T} = \{\mathbf{x}_i\}_{i = 1}^N$ with samples i.i.d. from $\mu$, we can approximate $\mu$ using the empirical distribution $\hat{\mu}_{\mathcal{T}} = \frac{1}{N} \sum_{i = 1}^N \delta_{\mathbf{x}_i}$. The expectation under the push-forward of this empirical distribution is then $\mathbb{E}_{\mathbf{z} \sim g^{\#} \hat{\mu}_{\mathcal{T}}} [\mathbf{z}] = \mathbb{E}_{\mathbf{x} \sim \hat{\mu}_{\mathcal{T}}} [g (\mathbf{x})] = \frac{1}{N} \sum_{i = 1}^N g (\mathbf{x}_i)$. The formulations given in \Cref{tab:change_space} can thus be replaced with their empirical counterparts, leading to practical optimization objectives commonly used in the dataset condensation (DC) literature. \Cref{tab:change_space_emp} summarizes these empirical formulations using integral probability metrics.
\begin{table}[th]
    \centering
    \makebox[\linewidth][c]{
        \begin{NiceTabular}{|c|c|c|}
            \toprule
            \diagbox{Match}{Optimize}& Input space $\mathbb{R}^n$ & Latent space $\mathbb{R}^m$ \\
            \midrule
            Input space $\mathbb{R}^n$ & $\displaystyle\min_{\red{\mathcal{S}} \subseteq \mathbb{R}^n} \sup_{\blue{h} \in \mathcal{H} (\mathbb{R}^n)} \big|\overline{\blue{h} (\mathcal{T})} - \overline{\blue{h} (\red{\mathcal{S}})}\big|$ & $\displaystyle\min_{\red{\mathcal{Z}} \subseteq \mathbb{R}^m} \sup_{\blue{h} \in \mathcal{H} (\mathbb{R}^n)} \big|\overline{\blue{h} (\mathcal{T})} - \overline{\blue{h} \circ g_d (\red{\mathcal{Z}})}\big|$ \\
            \midrule
            Latent space $\mathbb{R}^m$ & $\displaystyle\min_{\red{\mathcal{S}} \subseteq \mathbb{R}^n} \sup_{\blue{h} \in \mathcal{H} (\mathbb{R}^m)} \big|\overline{\blue{h} \circ g_e (\mathcal{T})} - \overline{\blue{h} \circ g_e (\red{\mathcal{S}})}\big|$ & $\displaystyle\min_{\red{\mathcal{Z}} \subseteq \mathbb{R}^n} \sup_{\blue{h} \in \mathcal{H} (\mathbb{R}^m)} \big|\overline{\blue{h} \circ g_e (\mathcal{T})} - \overline{\blue{h} (\red{\mathcal{Z}})}\big|$ \\
            \bottomrule
        \end{NiceTabular}
    }
    \caption{Regimes of optimizing and matching datasets in input and latent spaces.}
    \label{tab:change_space_emp}
\end{table}

In particular, let $\mathcal{H}(\mathbb{R}^m)$ be an RKHS with a universal kernel $k: \mathbb{R}^m \times \mathbb{R}^m \to \mathbb{R}$. \rev{Suppose the data domain $\Omega \subseteq \mathbb{R}^n$ is compact and the encoder $g_e$ is \emph{continuous and injective on $\Omega$}; then $g_e$ is a topological embedding of $\Omega$ onto the compact set $g_e (\Omega) \subseteq \mathbb{R}^m$, and if $k$ is universal on $g_e (\Omega)$, the pullback kernel $k_e := k \circ (g_e \otimes g_e)$ is universal (hence characteristic) on $\Omega$ \citep{li2017mmd, sriperumbudur2011universality}. Consequently,} matching two distributions $\mu_1$ and $\mu_2$ in a universal RKHS over the input space is equivalent to matching their push-forward distributions $g_e^{\#} \mu_1$ and $g_e^{\#} \mu_2$ in the corresponding RKHS over the latent space. \rev{We stress the scope of this statement: injectivity alone, without the compactness and continuity assumptions above, does not yield universality of the pullback kernel. However, by the Lusin--Suslin argument of the previous section, Borel injectivity of $g_e$ on the relevant supports already suffices for the weaker property that is actually needed in DC, which distinguishes the data distributions at hand.}

\section{Taxonomy for Dataset Condensation Methods}
Based on the discussions in \Cref{sec:dc}, we categorize existing DC methods according to the discrepancy metric used and the distribution matching and optimization spaces considered. Further improvements and variants can be derived by either augmenting the dataset or incorporating regularization terms into the main loss function as discussed in \Cref{sec:trick}. In this section, we organize the most popular DC approaches within the theoretical framework we have established. For each category or sub-category, we first introduce the key concept (e.g., the metric or the space), followed by a concise description of the algorithms proposed in the related works. Note that we use a consistent set of notations and terminologies throughout, which may differ from those in the original papers, to maintain alignment with the definitions introduced in the previous sections. \rev{Throughout this section, we use two reading aids. First, every method-to-framework correspondence is labeled with its \emph{type}: \textbf{exact} (the method's objective is exactly an instance of the stated discrepancy, possibly with an induced kernel), \textbf{estimator} (a finite-sample or sampled approximation of it), \textbf{upper bound}, \textbf{relaxation}, or \textbf{analogy} (a conceptual correspondence only); the labels are stated per method below and collected in \Cref{tab:overview}. Second, we separate survey material from new content: this section and \Cref{sec:trick} survey and organize prior work, while the paper's new formal claims are confined to \Cref{sec:pre,sec:dc,sec:main} (Propositions~\ref{prop:discrepancy}--\ref{prop:bridge}).}

\subsection{Discrepancy-based Classification} \label{sec:metric}
Given a probability distribution $\mu_1$, our goal is to find another distribution $\mu_2$ that is closest to $\mu_1$ under a specified discrepancy metric $D(\mu_1, \mu_2; \mathcal{H})$, where $\mathcal{H}$ denotes a hypothesis class. The optimization problem is formulated as:
$$\min_{\red{\mu_2}} \; D (\mu_1, \red{\mu_2}; \mathcal{H}).$$
In practice, we do not have access to the true distribution $\mu_1$\rev{; instead we observe a finite i.i.d.\ sample from it. The synthetic dataset, by contrast, is not a sample from an underlying population: it is the optimization variable itself, and its empirical measure \emph{is} the synthetic measure (cf.\ \Cref{sec:bridge})}. Let $\mathcal{T}$ be a dataset sampled from $\mu_1$ and $\mathcal{S}$ be \rev{the} synthetic dataset. We then approximate the discrepancy using $D(\hat{\mu}_{\mathcal{T}}, \hat{\mu}_{\mathcal{S}}; \mathcal{H})$, where $\hat{\mu}_{\mathcal{T}}$ and $\hat{\mu}_{\mathcal{S}}$ denote the empirical distributions associated with $\mathcal{T}$ and $\mathcal{S}$, respectively. For convenience, we define the empirical loss of a model $f$ on a dataset $\mathcal{T}$ as $\mathcal{L} (f, \mathcal{T}) := \mathbb{E}_{(\mathbf{x}, y) \sim \hat{\mu}_{\mathcal{T}}} [l(f(\mathbf{x}), y)] = \frac{1}{N} \sum_{i = 1}^N l(f(\mathbf{x}_i), y_i)$, where $l$ is a task-specific loss function.

\rev{\paragraph{A recurring identification: random-model averaging.} Many methods below replace a supremum over $\mathcal{H}$ by an \emph{average} over models $h \sim p$ drawn from a distribution over randomly initialized or pretrained networks. An average does not estimate a supremum. It is, however, an exact object of its own: assuming square-integrability of $h (z)$ under $p$,
$$\mathbb{E}_{h \sim p} \Big[\big(\overline{h (\mathcal{T})} - \overline{h (\mathcal{S})}\big)^2\Big] = \mathrm{MMD}^2 \big(\hat{\mu}_{\mathcal{T}}, \hat{\mu}_{\mathcal{S}}; k_p\big), \qquad k_p (z, z') := \mathbb{E}_{h \sim p} [h (z) \, h (z')],$$
i.e., the averaged objective equals the squared MMD with the induced random-feature kernel $k_p$ (for vector-valued outputs, sum over output coordinates). Where this applies, we label the objective \emph{exact (MMD with induced kernel)} rather than an estimator of an IPM. To see that average and supremum are genuinely different objects, take $p$ Gaussian over linear functions $h_{\mathbf{w}} (z) = \langle \mathbf{w}, z \rangle$: the induced kernel is $k_p (z, z') = \langle z, z' \rangle$ and the averaged objective reduces to squared mean matching $\|\overline{\mathcal{T}} - \overline{\mathcal{S}}\|_2^2$, whereas the supremum over $\|\mathbf{w}\|_2 \le 1$ is the mean distance $\|\overline{\mathcal{T}} - \overline{\mathcal{S}}\|_2$; for richer classes the two need not even induce the same minimizers. The iterated variants, in which $h$ is periodically re-trained on $\mathcal{S}$, fall outside this identity and are labeled heuristic estimators.}

\subsubsection{Generalization discrepancy} \label{sec:gd}
$$\boxed{\min_{\red{\mathcal{S}}} \; \mathrm{GD} (\hat{\mu}_{\mathcal{T}}, \hat{\mu}_{\red{\mathcal{S}}}; \mathcal{H}) \Longleftrightarrow \min_{\red{\mathcal{S}}} \; \mathcal{L} (h^*_{\red{\mathcal{S}}}, \mathcal{T}), \; h^*_{\red{\mathcal{S}}} \in \min_{\blue{h}} \; \mathcal{L} (\blue{h}, \red{\mathcal{S}})}$$
    
\paragraph{Backpropagation-Through-Time (BPTT) \citep{wang2018dataset}} Given a fixed original dataset $\mathcal{T}$, this method defines an updated loss function that no longer explicitly depends on $\mathcal{T}$: $\tilde{\mathcal{L}} (h, \eta, \mathcal{S}) := \mathcal{L} (h - \eta \cdot \frac{\partial \mathcal{L}}{\partial h} \big|_{(h, \mathcal{S})}, \mathcal{T})$. This formulation allows us to differentiate through one step of model training on synthetic data $\mathcal{S}$ and evaluate performance on the original data $\mathcal{T}$. The optimization is performed iteratively as follows: At iteration $k$, given current parameters $h^k$, learning rate $\eta^k$, and synthetic dataset $\mathcal{S}^k$, update:
    \begin{align*}
        & \eta^{k+1} \leftarrow \eta^k - \alpha \cdot \frac{\partial \tilde{\mathcal{L}}}{\partial \eta} \bigg|_{(h^k, \eta^k, \mathcal{S}^k)}, \; \mathcal{S}^{k+1} \leftarrow \mathcal{S}^k - \alpha \cdot \frac{\partial \tilde{\mathcal{L}}}{\partial \mathcal{S}} \bigg|_{(h^k, \eta^k, \mathcal{S}^k)} \\
        & h^{k+1} \leftarrow h^k - \eta^{k+1} \cdot \frac{\partial \mathcal{L}}{\partial h} \bigg|_{(h^k, \eta^{k+1}, \mathcal{S}^{k+1})}
    \end{align*}
    This iterative procedure decomposes the original bi-level optimization problem into two sub-problems: (1) updating the synthetic data and learning rate by differentiating through the inner optimization, and (2) updating the model using standard gradient descent on synthetic data.
    
    \paragraph{Convexified Implicit Gradients (CIG) \citep{loo2023dataset}} This method formulates the DC objective using a bi-level optimization framework. The outer loss is defined as $\tilde{\mathcal{L}} (h^*_{\mathcal{S}}, \mathcal{S}) := \mathcal{L} (h^*_{\mathcal{S}}, \mathcal{T})$, where $h^*_{\mathcal{S}}$ is the model trained to optimality on the synthetic dataset $\mathcal{S}$, and $\mathcal{L}(\cdot, \mathcal{T})$ is the empirical loss on the true dataset $\mathcal{T}$. Rather than relying on backpropagation through the inner optimization trajectory (as in BPTT), CIG employs implicit differentiation via the chain rule on the optimality conditions. At iteration $k$, compute surrogate optimal model $h^k$ (e.g., using kernel ridge regression, discussed in the next section). Then compute implicit gradient via the formula:
    \begin{align*}
        & v^k = \bigg(\frac{\partial^2 \mathcal{L}}{\partial h \partial h^T}\bigg)^{-1} \bigg|_{(h^k, \mathcal{S}^k)} \cdot \left (\frac{\partial \tilde{\mathcal{L}}}{\partial h}\right) \bigg|_{(h^k, \mathcal{S}^k)} \\
        & \nabla g (\mathcal{S}^k) = \frac{\partial \tilde{\mathcal{L}}}{\partial \mathcal{S}} \bigg|_{(h^k, \mathcal{S}^k)} + \frac{\partial}{\partial \mathcal{S}} \bigg(\bigg(\frac{\partial \mathcal{L}}{\partial h}\bigg)^T \bigg|_{(h^k, \mathcal{S}^k)} \cdot v^k\bigg)
    \end{align*}
    Finally update synthetic dataset using the gradient: $\mathcal{S}^{k+1} \leftarrow \mathcal{S}^k - \alpha \cdot \nabla g (\mathcal{S}^k)$. The key distinction from BPTT is that CIG avoids differentiating through the full training trajectory by exploiting the optimality condition $\partial \mathcal{L}(h^*_{\mathcal{S}}, \mathcal{S}) / \partial h = 0$ and applying the implicit function theorem. This leads to a computationally efficient and stable surrogate for the gradient of the outer objective with respect to $\mathcal{S}$.

\subsubsection{Value discrepancy}
$$\boxed{\min_{\red{\mathcal{S}}} \; \mathrm{VD} (\hat{\mu}_{\mathcal{T}}, \hat{\mu}_{\red{\mathcal{S}}}; \mathcal{H}) = \|h^*_{\red{\mathcal{S}}} - h^*_{\mathcal{T}}\|, \; h^*_{\red{\mathcal{S}}} \in \min_{\blue{h}} \; \mathcal{L} (\blue{h}, \red{\mathcal{S}}), \; h^*_{\mathcal{T}} \in \min_{\blue{h}} \; \mathcal{L} (\blue{h}, \mathcal{T})}$$
    \paragraph{Feature matching \citep{wang2022cafe}} Let $h^*_{\mathcal{S}}$ denote the model that minimizes the empirical loss over the synthetic dataset $\mathcal{S}$, i.e.,  $h^*_{\mathcal{S}} \in \min_h \; \mathcal{L} (h, \mathcal{S})$. The goal is to minimize the discrepancy between the feature representations of the real and synthetic datasets, as evaluated by the optimal model $h^*_{\mathcal{S}}$. Specifically, we compare the layer-wise mean features of the original dataset $\mathcal{T}$ and the synthetic dataset $\mathcal{S}$, and define the objective:
    $$\mathcal{S}^* \in \arg \min_{\red{\mathcal{S}}} \sum_{l = 1}^L \big\|\overline{(h^*_{\red{\mathcal{S}}})_l (\mathcal{T})} - \overline{(h^*_{\red{\mathcal{S}}})_l (\red{\mathcal{S}})} \big\|_2^2 + L_{dis}, \; h^*_{\red{\mathcal{S}}} \in \min_{\blue{h}} \; \mathcal{L} (\blue{h}, \red{\mathcal{S}}),$$
    where $L$ is the number of layers in the model $h$, $(h)_l(\cdot)$ denotes the activation (feature) output of layer $l$ on \rev{a dataset}, $\overline{(\cdot)}$ indicates averaging over the dataset, and $\mathcal{L}_{\mathrm{dis}}$ is a discrimination loss that encourages class separability (as described in \Cref{sec:loss_reg}). The above is a bi-level optimization problem, where the inner optimization finds $h^*_{\mathcal{S}}$ and the outer optimization updates $\mathcal{S}$. It is solved using backpropagation-through-time (BPTT), as discussed in \Cref{sec:gd}. \rev{\emph{Correspondence: analogy (surrogate for VD).} Strictly, this objective matches layer-wise feature \emph{statistics} under the single model $h^*_{\mathcal{S}}$, i.e., a model-specific statistic discrepancy, and is not the value discrepancy $\|h^*_{\mathcal{S}} - h^*_{\mathcal{T}}\|$ defined above: no second model trained on $\mathcal{T}$ appears, and averaged intermediate features do not determine the function. We classify it here because it is used in practice as a surrogate for functional closeness, and we label it accordingly in \Cref{tab:overview}.}

\subsubsection{Parameter discrepancy}
$$\boxed{\min_{\red{\mathcal{S}}} \; \mathrm{PD} (\hat{\mu}_{\mathcal{T}}, \hat{\mu}_{\red{\mathcal{S}}}; \mathcal{H}) = \|\theta^*_{\red{\mathcal{S}}} - \theta^*_{\mathcal{T}}\|, \; \theta^*_{\red{\mathcal{S}}} \in \min_{\blue{\theta}} \; \mathcal{L} (h_{\blue{\theta}}, \red{\mathcal{S}})}$$
    \paragraph{Trajectory matching \citep{cazenavette2022dataset}} This approach optimizes a synthetic dataset so that training on it reproduces segments of the training dynamics observed on the real dataset. Let $\{\theta_{t, \mathcal{T}}\}_{t=0}^T$ denote a precomputed training trajectory obtained on the real dataset $\mathcal{T}$. Instead of training from random initialization, the model trained on the synthetic dataset $\mathcal{S}$ is initialized from selected checkpoints along this real trajectory. Specifically, for a chosen starting point $\theta_{t, \mathcal{T}}$, the parameters are updated for several steps using $\mathcal{S}$, yielding $\theta_{t+\tau, \mathcal{S}}$. The objective is then to match the resulting parameters to the corresponding point on the real trajectory:
    $$\mathcal{S}^* \in \arg \min_{\red{\mathcal{S}}} \sum_{t = 0}^{T} \sum_{\tau = 1}^{K} \|\theta^*_{t+\tau, \red{\mathcal{S}}} - \theta^*_{t+\tau, \mathcal{T}}\|, $$
    where $\theta_{t+\tau, \mathcal{S}}$ is obtained by training on $\mathcal{S}$ starting from $\theta_{t, \mathcal{T}}$. This formulation mitigates the ill-posedness of directly matching independently trained parameter trajectories in over-parameterized networks, where symmetries (e.g., neuron permutations) and flat regions in the loss landscape can lead to large parameter discrepancies despite similar functional behavior. The resulting bi-level optimization problem is solved using backpropagation-through-time (BPTT), as described in \Cref{sec:gd}. \rev{\emph{Correspondence: relaxation (coupled symmetry-aligned PD).} By anchoring the synthetic-data run at checkpoints of the real trajectory, this construction implements exactly the run-coupling required by assumption (c) of the $\mathrm{PD}_{\mathcal{G}}$ hierarchy in \Cref{sec:dc}; the matched quantity is a sum of coupled parameter distances along the trajectory rather than the endpoint aligned distance itself, hence a relaxation.}

\subsubsection{Characteristic Discrepancy}
$$\boxed{\min_{\red{\mathcal{S}}} \; \mathrm{CD} (\hat{\mu}_{\mathcal{T}}, \hat{\mu}_{\red{\mathcal{S}}}; \mathcal{H})}$$
\paragraph{Dataset distillation with neural characteristic function \citep{wang2025dataset}} This approach formulates dataset condensation by minimizing a discrepancy between the empirical distributions of real and synthetic data using \emph{characteristic functions}. Specifically, let $CD$ denote a characteristic discrepancy defined over a function class $\mathcal{H}$ that approximates complex exponentials via neural networks. The objective becomes:
$$
\mathcal{S}^* \in \arg \min_{\red{\mathcal{S}}} \; \mathbb{E}_{\omega \sim p(\omega)} \big\| \overline{h_{\omega}(\mathcal{T})} - \overline{h_{\omega}(\red{\mathcal{S}})} \big\|_2^2,
$$
where $h_{\omega} \in \mathcal{H}$ is a neural network parameterization of complex-valued features. Unlike standard IPMs, characteristic discrepancies leverage the fact that characteristic functions uniquely determine probability distributions. By learning a flexible parameterization of these functions, the method captures higher-order statistics and global structure of the data distribution. In practice, the expectation over $\omega$ is approximated using a finite set of sampled frequencies, and the neural parameterization enables scalable and expressive matching beyond fixed kernels. \rev{\emph{Correspondence: exact (MMD with learned spectral measure).} With frequency distribution $p (\omega)$ and complex-exponential features, the objective is the squared $\mathrm{CD}_{L_2 (p)}$, which by Bochner's theorem equals the squared MMD with the translation-invariant kernel whose spectral measure is $p$ (\Cref{sec:dc}); learning the parameterization amounts to learning the spectral measure.}

\subsubsection{Distribution discrepancy}
$$\boxed{\min_{\red{\mathcal{S}}} \; \mathrm{DD} (\hat{\mu}_{\mathcal{T}}, \hat{\mu}_{\red{\mathcal{S}}}; \mathcal{H})}$$
    \paragraph{Distribution matching \citep{zhao2023dataset}} This approach formulates dataset condensation as minimizing a distributional discrepancy between real and synthetic datasets, measured via an IPM. Specifically, let $\mathrm{DD} = \mathrm{IPM}$ and consider \rev{a uniformly bounded function class $\mathcal{H}$ (an IPM over all measurable functions is infinite whenever the distributions differ; the bounded unit-ball choice corresponds to total variation, cf.\ \Cref{sec:pre})}. The objective becomes:
    $$\mathcal{S}^* \in \arg \min_{\red{\mathcal{S}}} \sup_{h \in \mathcal{H}} \big\|\overline{h(\mathcal{T})} - \overline{h(\red{\mathcal{S}})} \big\|_2^2.$$
    In practice, the supremum over $\mathcal{H}$ is \rev{replaced by an average} over a batch of randomly initialized or pretrained models $h$ sampled from the function class. \rev{\emph{Correspondence: exact (MMD with the induced random-feature kernel $k_p$), by the identification at the start of this section; the averaged objective is not an estimator of the supremum.}}
    
    \paragraph{Gradient matching \citep{zhao2021dataset}} This method treats dataset condensation as a distribution alignment problem, using an IPM defined over gradients of the loss function. Specifically, let $\mathrm{DD} = \mathrm{IPM}$ and let $\mathcal{H} = \mathcal{C}^1$ be the space of continuously differentiable functions. The objective becomes:
    $$\mathcal{S}^* \in \arg \min_{\red{\mathcal{S}}} \sup_{h \in \mathcal{H}} \big\|\overline{\nabla h (\mathcal{T})} - \overline{\nabla h (\red{\mathcal{S}})} \big\|_2^2.$$
    Here, $\nabla h(\mathcal{T})$ and $\nabla h(\mathcal{S})$ denote the gradients of the loss with respect to model parameters, evaluated on real and synthetic datasets, respectively. The goal is to match the average gradient signal induced by the two datasets. In practice, the supremum over $\mathcal{H}$ is \rev{replaced by an average} over gradients computed from a batch of randomly initialized or pretrained models $h$. \rev{\emph{Correspondence: exact in expectation (one-step, random models) / heuristic estimator (iterated).} Two clarifications are in order. First, gradients are vector-valued and label-dependent, so the objective is not directly an IPM over $\mathcal{C}^1$; the explicitly indexed scalar test class is $\mathcal{F}_{\nabla} := \{z \mapsto \langle \mathbf{u}, \nabla_{\theta} \ell (z; \theta) \rangle: \|\mathbf{u}\|_2 \le 1, \, \theta \in \Theta\}$, with the Euclidean norm on parameter space. Second, in the one-step form with random parameters $\theta \sim p$, the expected objective satisfies the induced-kernel identity with $k (z, z') = \mathbb{E}_{\theta \sim p} \langle \nabla_{\theta} \ell (z; \theta), \nabla_{\theta} \ell (z'; \theta) \rangle$, which is an expected (empirical) neural-tangent-type kernel, and is therefore exactly a squared MMD. The practical iterated variant, in which $\theta$ is updated on $\mathcal{S}$ between synthetic-data updates, falls outside this identity and is a heuristic estimator.}

    \paragraph{Higher-order moment matching \citep{yu2025teddy}} While standard distribution matching aligns only the first-order moments (means) of features extracted from the real and synthetic datasets, higher-order moment matching extends this idea by also aligning second-order moments (covariances or variances). Specifically, for a function class $\mathcal{H}$ and using an integral probability metric (IPM), the objective becomes:
    $$\mathcal{S}^* \in \arg \min_{\red{\mathcal{S}}} \sup_{h \in \mathcal{H}} \big\|\overline{h(\mathcal{T})} - \overline{h(\red{\mathcal{S}})} \big\|_2^2 + \big\|Var (h(\mathcal{T})) - Var (h(\red{\mathcal{S}})) \big\|_2^2.$$
    Here, $\overline{h(\cdot)}$ denotes the empirical mean of features and $\text{Var}(h(\cdot))$ denotes the feature-wise variance under function $h$. Interestingly, \citep{yu2025teddy} show that gradient matching is a special case of this approach, which implicitly aligns both first- and second-order moments. In practice, the supremum over the function class $\mathcal{H}$ is \rev{replaced by an average} over a batch of randomly initialized or pretrained models $h$. \rev{\emph{Correspondence: relaxation (lifted moment matching).} Variance matching is nonlinear in the underlying measure and is not itself an IPM as written; it is first- and second-moment matching of the \emph{lifted statistic} $z \mapsto (h (z), h (z)^2)$, and the stated connection to gradient matching is retained in this lifted-statistic form.}
    
    \paragraph{M3D \citep{zhang2024m3d}} This method formulates dataset condensation as a distribution matching problem using the maximum mean discrepancy (MMD) as the discrepancy metric. Specifically, let $\mathcal{H}$ be an RKHS associated with a kernel $k$. Then the objective is:
    $$\mathcal{S}^* \in \arg \min_{\red{\mathcal{S}}} \sup_{\|\blue{h}\|_{\mathcal{H}} \le 1} \big\|\overline{\blue{h} (\mathcal{T})} - \overline{\blue{h} (\red{\mathcal{S}})} \big\|_2^2.$$
    This corresponds exactly to minimizing the MMD between the empirical distributions of $\mathcal{T}$ and $\mathcal{S}$. In practice, common choices for the kernel $k$ include the Gaussian kernel, Laplacian kernel, and the neural tangent kernel (NTK). The MMD is then computed using the kernel embedding formulation described in \Cref{sec:pre}, enabling efficient and differentiable optimization of the synthetic dataset. \rev{\emph{Correspondence: exact (MMD).}}
    
    \paragraph{Wasserstein distance \citep{liu2023dataset}} This approach uses the Wasserstein distance as the discrepancy measure to align the empirical distributions of the original and synthetic datasets. Specifically, let $\mathrm{DD} = \mathrm{WM}$ be the Wasserstein metric and $\mathcal{H}$ be the set of all 1-Lipschitz functions. The optimization problem becomes:
    $$\mathcal{S}^* \in \arg \min_{\red{\mathcal{S}}} \sup_{\blue{h} \in \mathcal{H}} \big\|\overline{\blue{h} (\mathcal{T})} - \overline{\blue{h} (\red{\mathcal{S}})} \big\|_2^2.$$
    By Kantorovich-Rubinstein duality, this \rev{relates to} the 1-Wasserstein distance between the empirical distributions \rev{$\hat{\mu}_{\mathcal{T}}$ and $\hat{\mu}_{\mathcal{S}}$}. In practice, the dual formulation reduces to a tractable linear program (LP) that can be efficiently solved, enabling gradient-based optimization of the synthetic dataset $\mathcal{S}$. \rev{\emph{Correspondence: relaxation.} The Kantorovich--Rubinstein dual of $\mathrm{WM}_1$ is $\sup_{\mathrm{Lip} (h) \le 1} |\overline{h (\mathcal{T})} - \overline{h (\mathcal{S})}|$; the squared, batched objective displayed above is a practical relaxation of this dual, and the duality itself is specific to $p = 1$ (\Cref{sec:pre}).}
    
    \paragraph{Minimum finite covering \citep{chen2024adversarial}} This approach uses the Hausdorff distance as the discrepancy metric between the original dataset $\mathcal{T} = {(\mathbf{x}_i, y_i)}_{i=1}^N$ and the synthetic dataset $\mathcal{S} = {(\tilde{\mathbf{x}}_j, \tilde{y}_j)}_{j=1}^M$, which requires no explicit function space $\mathcal{H}$. The condensation problem is formulated as:
    $$\mathcal{S}^* \in \arg \min_{\red{\mathcal{S}}} d_H (\mathcal{T}, \red{\mathcal{S}}),$$
    where the Hausdorff distance is defined by 
    $$d_H (\mathcal{T}, \mathcal{S}) = \max \big\{\max_{\red{i}} \min_{\blue{j}} d(\mathbf{x}_{\red{i}}, \tilde{\mathbf{x}}_{\blue{j}}), \max_{\red{j}} \min_{\blue{i}} d(\mathbf{x}_{\blue{i}}, \tilde{\mathbf{x}}_{\red{j}})\big\}.$$ 
    For tasks such as coreset selection, minimizing the Hausdorff distance can be formulated as a mixed integer linear programming (MILP) problem, allowing for exact or approximate combinatorial optimization of the synthetic subset. \rev{\emph{Correspondence: analogy (support-level geometric criterion).} The Hausdorff distance compares supports (finite point sets) and ignores probability mass entirely; it is not a probability metric and vanishing Hausdorff distance does not imply equal measures. Its relation to DC is through covering-type control of the supports where small $d_H$ means every real point has a nearby synthetic point and vice versa, which is the sense in which we include it here.}

\rev{\subsubsection{Uni-level (decoupled) formulations} \label{sec:unilevel}
$$\boxed{h^*_{\mathcal{T}} \in \arg\min_{\blue{h}} \; \mathcal{L} (\blue{h}, \mathcal{T}), \qquad \mathcal{S}^* \in \arg\min_{\red{\mathcal{S}}} \; D_{\mathrm{stat}} \big(\mathcal{T}, \red{\mathcal{S}}; h^*_{\mathcal{T}}\big), \qquad \tilde{y}_j = h^*_{\mathcal{T}} (\tilde{\mathbf{x}}_j)}$$
A recent line of work removes the bi-level structure of \eqref{eq:dc} entirely by decoupling model training from data synthesis; because these methods never differentiate through training, they scale to full-resolution ImageNet-scale datasets, and they occupy a branch of their own in \Cref{tab:overview}.

\paragraph{Squeeze, Recover and Relabel (SRe$^2$L) \citep{yin2023squeeze}} The method proceeds in three decoupled stages, each of which maps directly onto our framework. \emph{Squeeze} trains a single model $h^*_{\mathcal{T}}$ once on the full dataset. \emph{Recover} synthesizes $\mathcal{S}$ by matching statistics under this \emph{fixed} model: the classification loss of $h^*_{\mathcal{T}}$ on the synthetic images together with the running BatchNorm statistics accumulated on $\mathcal{T}$,
$$\mathcal{S}^* \in \arg \min_{\red{\mathcal{S}}} \; \mathcal{L} (h^*_{\mathcal{T}}, \red{\mathcal{S}}) + \sum_{l} \Big(\big\|\bar{\mathbf{m}}_l (\red{\mathcal{S}}) - \mathbf{m}_l^{\mathrm{BN}}\big\|_2^2 + \big\|\bar{\mathbf{v}}_l (\red{\mathcal{S}}) - \mathbf{v}_l^{\mathrm{BN}}\big\|_2^2\Big),$$
where $\bar{\mathbf{m}}_l$ and $\bar{\mathbf{v}}_l$ denote the mean and variance of the layer-$l$ activations on $\mathcal{S}$, and $\mathbf{m}_l^{\mathrm{BN}}, \mathbf{v}_l^{\mathrm{BN}}$ the BatchNorm statistics of $h^*_{\mathcal{T}}$. In our terms, this is a \emph{model-specific statistic discrepancy}: a lifted (first- and second-moment) matching under the single model $h^*_{\mathcal{T}}$, in the same sense as the variance-matching discussion above, optimized \emph{single-level}. \emph{Relabel} assigns soft labels $\tilde{y}_j = h^*_{\mathcal{T}} (\tilde{\mathbf{x}}_j)$, connecting directly to the soft-label extension of Remark~\ref{rem:soft} and the joint view $\mu (\mathbf{x}, \tilde{y})$. \emph{Correspondence: exact (model-specific statistic discrepancy, single-level).}

\paragraph{Diversity-driven decoupled distillation (RDED) \citep{sun2024diversity}} A follow-up in the same decoupled paradigm that avoids synthesis-time optimization altogether: high-confidence real image patches (as scored by $h^*_{\mathcal{T}}$) are selected and composed into synthetic images, which are then relabeled by $h^*_{\mathcal{T}}$. In our framework this is a coreset-style support construction combined with the relabeling map, again single-level. \emph{Correspondence: exact (selection + relabeling under $h^*_{\mathcal{T}}$).}}

\subsection{Space-based Classification} \label{sec:space}
In \Cref{sec:metric}, we discussed how to practically adapt various discrepancy metrics within the theoretical framework of dataset condensation. However, we have not yet considered variations arising from changes in the matching or optimization spaces. Recall that, in the general formulation \eqref{eq:dc}, the outer optimization is performed over the space of distributions (i.e., the data space), while the inner optimization is over the hypothesis space (i.e., the model space). Consequently, variants of the DC methods introduced in \Cref{sec:metric} can be derived by altering either the data space or the model space.

\subsubsection{Data space}
For the data space, we can move from the original high-dimensional input space to a lower-dimensional latent space using generative models, as discussed in \Cref{sec:gen}. It is important to note that, in practice, an encoder mapping from the high-dimensional space to the latent space is typically not injective, and likewise, a decoder from the latent space back to the high-dimensional space is generally not surjective. As a result, the push-forward metric induced by the generative model no longer satisfies all the properties of a true metric. Nevertheless, this approach can still be effectively used to approximate the data distribution within the framework derived in \Cref{sec:main}.
    \paragraph{Distribution/Gradient matching with pre-trained GANs \citep{zhao2022synthesizing}} This approach synthesizes a condensed dataset by optimizing a set of latent vectors $\mathcal{Z}$, which are then passed through a pre-trained Generative Adversarial Network (GAN) generator $g_d$. Let $\mathcal{T}$ represent the original dataset. The optimal latent set $\mathcal{Z}^*$ found by minimizing the difference between the feature representations of the original and synthetic data, a process known as distribution or gradient matching:
    $$\mathcal{S}^* = g_d (\mathcal{Z}^*), \; \mathcal{Z}^* \in \arg \min_{\red{\mathcal{Z}}} \sup_{h \in \mathcal{H}} \big\|\overline{h(\mathcal{T})} - \overline{h \circ g_d (\red{\mathcal{Z}})} \big\|_2^2.$$
    Here, $\mathcal{H}$ is a function space of feature extractors. In practice, the supremum over $\mathcal{H}$ approximated by averaging the outputs from a batch of pre-trained models $h$. 
    
    \paragraph{Distribution matching with GAN training \citep{wang2025dim}} Let $\mathcal{T}$ denote the original dataset, and $g_d$ be the generator of a GAN. The condensed dataset is generated by
    $$\mathcal{S}^* = g_d^* (\mathcal{Z}), \; g_d^* \in \arg \min_{\red{g_d}} \bigg(\underbrace{\sup_{h \in \mathcal{H}} \big\|\overline{h(\mathcal{T})} - \overline{h \circ \red{g_d} (\mathcal{Z})} \big\|_2^2}_{L_{\mathcal{Z}}} + L_{GAN} \bigg),$$
    where $\mathcal{Z}$ is a fixed set of random latent codes, and $L_{GAN}$ is the standard min-max loss of training a conditional GAN with generator $g_d$ and a discriminator $g_c$:
    $$L_{GAN} = \mathbb{E}_{\mathbf{x} \sim \mu} [\log g_c (\mathbf{x} | y)] + \mathbb{E}_{\mathbf{z} \sim \nu} [\log (1 - g_c (g_d (\mathbf{z} | y)))].$$
    The training proceeds by first pretraining $g_d$ through minimizing $L_{GAN}$ for several epochs to obtain a vanilla GAN. Then, the distribution matching loss $L_{\mathcal{Z}}$ is incorporated to further update $g_d$, encouraging the generator not only to produce realistic images but also to generate synthetic data that improves downstream model training performance. In practice, the supremum over the function space $\mathcal{H}$ is approximated by taking averaging over a batch of random initialized models $h$.
    
    \paragraph{Feature matching with GAN training \citep{zhang2023dataset}} Let $\mathcal{T}$ be the original dataset, $\mathcal{Z}$ be the synthetic latent set, and $g_d$ be the GAN generator. The condensed dataset is generated as
    $$\mathcal{S}^* = g_d^* (\mathcal{Z}^*), \; (\mathcal{Z}^*, g_d^*) \in \arg \min_{\red{g_d}} \bigg(\underbrace{\sum_{l = 1}^L \big\|\overline{h_l (\mathcal{T})} - \overline{h_l \circ \red{g_d} (\mathcal{Z})} \big\|_2^2}_{L_{\mathcal{Z}}} + L_{GAN} + L_* \bigg),$$
    where $h_l$ denotes the intermediate feature maps extracted from layer $l$ of a neural network $h$, and $L$ is the total number of such layers considered. The network $h$ itself is updated every few epochs during the alternating updates of $\mathcal{Z}$ and $g_d$. The additional regularization terms collected in $L_*$ include intra-class diversity loss $L_{intra}$ and inter-class discrimination loss $L_{inter}$, as described in \Cref{sec:loss_reg}.

    \paragraph{Feature matching with diffusion model training \citep{gu2024efficient}} First, a latent diffusion model (LDM) is trained by optimizing its decoder $g_d^*$ and denoiser $g_{dn}^*$ to minimize the total loss:
    $$L_{total} = L_{diff} + L_*,$$
    where $L_{diff}$ is the primary diffusion training loss, and $L_*$ includes regularization terms such as the representative loss $L_{rep}$ and diversity loss $L_{div}$, as discussed in \Cref{sec:loss_reg}. After training, the condensed dataset is generated by
    $$\mathcal{S}^* = (g_d^* \circ g_{dn}^*) (\mathcal{Z}),$$
    where $\mathcal{Z}$  is a fixed set of random latent noise vectors.
    
    \paragraph{General DC method with random or pre-trained GAN \citep{cazenavette2023generalizing}} Let $\mu_1$ denote the original data distribution, $\nu_2$ be the synthetic distribution in the latent space, and $g_d$ be the GAN generator. The synthetic distribution in data space is defined as a push-forward $\mu_2^* = g_d^{\#} \nu_2^*$, where the optimal latent distribution $\nu_2^*$ is obtained by minimizing a discrepancy metric $D$ between $\mu_1$ and $\mu_2 = g_d^{\#} \nu_2$, i.e., 
    $$\mu_2^* = g_d^{\#} \nu_2^*, \; \nu_2^* = \arg \min_{\red{\nu_2}} D (\mu_1, g_d^{\#}\red{\nu_2}; \mathcal{H}).$$
    In practice, $g_d$ can be either randomly initialized or pretrained, and the choice of discrepancy metric $D$ is flexible and can correspond to any dataset condensation metric discussed in previous sections.
    
    \paragraph{General DC methods with pre-trained diffusion model \citep{moser2024latent}} Let $\mu_1$ be the original data distribution, $\nu_2$ be the synthetic distribution in the latent space, and $g_d, g_{dn}$ be the decoder and denoiser components of an LDM. The synthetic data distribution is given by the push-forward $\mu_2^* = (g_d \circ g_{dn})^{\#} \nu_2^*$, where the optimal latent distribution $\nu_2^*$ is obtained by minimizing a discrepancy metric $D$ between $\mu_1$ and $\mu_2 = (g_d \circ g_{dn})^{\#} \nu_2$, that is, 
    $$\mu_2^* = (g_d \circ g_{dn})^{\#} \nu_2^*, \; \nu_2^* = \arg \min_{\red{\nu_2}} D (\mu_1, (g_d \circ g_{dn})^{\#}\red{\nu_2}; \mathcal{H}).$$
    In practice, the LDM is pre-trained, with the denoiser $g_{dn}$ typically implemented by a U-Net architecture. The discrepancy metric $D$ can be chosen from any dataset condensation metric introduced earlier. 
    
    \paragraph{General DC method with pretrained hallucinator-extractor \citep{liu2022dataset}} Let $\mu_1$ denote the original data distribution, $\nu_2$ be the synthetic distribution in latent space, and $\mathcal{H}$ a hypothesis space of hallucinators (learned data generators). The synthetic distribution is obtained by pushing forward $\nu_2$ through a hallucinator $h \in \mathcal{H}$, i.e.,
    $$\mu_2^* = (h^*)^{\#} \nu_2^*, \; (h^*, \nu_2^*) = \arg \min_{\red{h} \in \mathcal{H}, \; \blue{\nu_2} \in \mathcal{P}} D (\mu_1, \red{h}^{\#}\blue{\nu_2}; \mathcal{H}) + L_*,$$
    where $\mathcal{P}$ denotes the space of probability distributions over the latent space. The regularization term $L_*$ includes task-specific losses (e.g., cross-entropy for classification), a contrastive loss $L_{con}$, and a cosine similarity loss $L_{cos}$, as detailed in \Cref{sec:loss_reg}.

    \paragraph{$K$-means clustering with pretrained diffusion model \citep{su2024d4m}} Let $\mathcal{T}$ denote the original dataset, and $cl$ represent the operator that computes the centers of $K$-means clustering. We utilize a pretrained LDM with an encoder $g_e^*$, decoder $g_d^*$, and denoiser $g_{dn}^*$. The condensed dataset is generated as follows:
    $$\mathcal{S}^* = (g_d^* \circ g_{dn}^*) (cl(g_e^* (\mathcal{T}))).$$
    The process involves encoding the original dataset $\mathcal{T}$ into a low-dimensional latent space using the pretrained encoder $g_e^*$, applying $K$-means clustering to select $K$ representative centroids in this latent space, and then reconstructing the synthetic dataset $\mathcal{S}^*$ through the pretrained decoder $g_d^*$ and denoiser $g_{dn}^*$. Notably, the model is not trained on the synthetic dataset by minimizing a loss between predictions and fixed labels. Instead, it leverages soft labels from a distilled student network, optimizing the parameters $\theta_S^*$ as:
    $$\theta_S^* \in \arg \min_{\theta_S} L (f_{\theta_T} (\mathbf{x}), f_{\theta_S} (\mathbf{x})),$$
    where $f_{\theta_T}$ is the teacher network and $f_{\theta_S}$ is the student model trained on the synthetic dataset.

\subsubsection{Model space}
For model space, our focus is on the training trajectory and kernel embedding for dataset compression, driven by two key motivations. First, although model parameters reside in a high-dimensional space, the updates during training typically span a low-dimensional subspace, enabling efficient representation of the training dynamics. Second, by leveraging kernel embedding, we can transform feature representations to effectively condense the dataset, capturing essential information in a compact form.

    \paragraph{Distribution matching with expert subspace projection \citep{ma2023dataset}} Let $\mathcal{T}$ represent the original dataset, $\mathcal{S}$ the synthetic dataset, and $h_{\theta^*} \in \mathcal{H}$ a teacher model with training trajectory $\Theta = \{\theta^*_t\}_{t = 1}^T$, where $\mathbb{S}$ denotes the linear subspace spanned by $\Theta$. The condensed dataset $\mathcal{S}^*$ is obtained by solving:
    $$\mathcal{S}^* \in \arg \min_{\red{\mathcal{S}}} \sup_{\blue{\theta}} \big\|\overline{h_{\blue{\theta}} (\mathcal{T})} - \overline{h_{\blue{\theta}} (\red{\mathcal{S}})} \big\|_2^2 + L_{proj} (\blue{\theta}).$$
    where $L_{proj} (\theta) = \|\theta - \text{Proj}_{\mathbb{S}} (\theta)\|_1$ measures the $\ell_1$-distance between the model parameter $\theta$ and its projection onto $\mathbb{S}$. In practice, the supremum over the function space $\mathcal{H}$ is approximated by averaging over a batch of randomly initialized or pretrained models $h$.

    \paragraph{Kernel Ridge regression with NTK \citep{nguyen2021dataset-a, nguyen2021dataset-b}} Given a real dataset $\mathcal{T} = {(\mathbf{x}_i, y_i)}_{i = 1}^N$ and a synthetic dataset $\mathcal{S} = {(\tilde{\mathbf{x}}_j, \tilde{y}_j)}_{j = 1}^M$, the solution to linear ridge regression over $\mathcal{S}$ is obtained by solving:
    \begin{align*} \label{ridge}
        \mathbf{w}_{\mathcal{S}} \in \arg\min_{\mathbf{w} \in \mathbb{R}^n} \; \frac{1}{2} \sum_{j = 1}^M \|\mathbf{w}^T \tilde{\mathbf{x}}_j - \tilde{y}_j\|_2^2 + \frac{1}{2} \lambda \|\mathbf{w}\|_2^2. \tag{Ridge}
    \end{align*}
    Let $\tilde{\mathbf{X}}_{\mathcal{S}} = \begin{bmatrix}
        \tilde{\mathbf{x}}_1 & \cdots & \tilde{\mathbf{x}}_M
    \end{bmatrix} \in \mathbb{R}^{n \times M}$ denote the matrix whose columns are the synthetic feature vectors $\tilde{\mathbf{x}}_j$, and let $\tilde{\mathbf{y}}_{\mathcal{S}} \in \mathbb{R}^M$ be the corresponding label vector. Then the closed-form solution of problem \eqref{ridge} is:
    $$\mathbf{w}_{\mathcal{S}} = \tilde{\mathbf{X}}_{\mathcal{S}} \big(\tilde{\mathbf{X}}_{\mathcal{S}}^T \tilde{\mathbf{X}}_{\mathcal{S}} + \lambda \mathbf{I}_M\big)^{-1} \tilde{\mathbf{y}}_{\mathcal{S}}.$$
    To evaluate the quality of the synthetic dataset $\mathcal{S}$, one minimizes the predictive error of $\mathbf{w}_{\mathcal{S}}$ on the original dataset $\mathcal{T}$: 
    $$\mathcal{S}^* \in \arg \min_{\mathcal{S}} \sum_{i = 1}^N \|y_i - \mathbf{w}^T_{\mathcal{S}} \mathbf{x}_i\|_2^2 = \sum_{i = 1}^N \|y_i - \mathbf{x}_i^T \tilde{\mathbf{X}}_{\mathcal{S}} \big(\tilde{\mathbf{X}}_{\mathcal{S}}^T \tilde{\mathbf{X}}_{\mathcal{S}} + \lambda \mathbf{I}_M\big)^{-1} \tilde{\mathbf{y}}_{\mathcal{S}}\|_2^2.$$
    This framework can be extended to nonlinear models via kernel methods. Let $k: \mathbb{R}^n \times \mathbb{R}^n \to \mathbb{R}$ be a positive-definite kernel function, and define the kernel matrices: $K_{\mathcal{S}, \mathcal{S}} \in \mathbb{R}^{M \times M}$ with entries $k(\tilde{\mathbf{x}}_j, \tilde{\mathbf{x}}_{j'})$, $K_{\mathcal{T}, \mathcal{S}} \in \mathbb{R}^{N \times M}$ with entries $k(\mathbf{x}_i, \tilde{\mathbf{x}}_j)$. Then the kernel ridge regression predictor trained on $\mathcal{S}$ is:
    $$h^*_{\mathcal{S}} (\mathbf{x}) = \sum_{j = 1}^M \alpha_j k(\mathbf{x}, \tilde{\mathbf{x}}_j), \; \text{with } \alpha = \big(K_{\mathcal{S}, \mathcal{S}} + \lambda \mathbf{I}_M\big)^{-1} \tilde{\mathbf{y}}_{\mathcal{S}}$$
    The predicted labels on $\mathcal{T}$ are then given by $K_{\mathcal{T}, \mathcal{S}} \big(K_{\mathcal{S}, \mathcal{S}} + \lambda \mathbf{I}_M\big)^{-1} \tilde{\mathbf{y}}_{\mathcal{S}}$. Hence, the optimal synthetic dataset minimizes the empirical loss on $\mathcal{T}$ under the kernel predictor:
    $$\mathcal{S}^* \in \arg \min_{\red{\mathcal{S}}} \mathcal{L} (h^*_{\red{\mathcal{S}}}, \mathcal{T}),$$
    where the loss is typically mean squared error.
    
    \paragraph{Kernel Ridge regression with neural network Gaussian process (NNGP) kernel \citep{loo2022efficient}} Let $\mathcal{T}$ denote the original dataset and $\mathcal{S}$ represent the synthetic dataset. We define a kernel function $k$ with feature maps $\varphi (\mathbf{x}) = \frac{1}{\sqrt{p}} [h_1 (\mathbf{x}), \ldots, h_p (\mathbf{x})]^T$, where each $h_i$ is drawn from a Gaussian process. The condensed dataset $\mathcal{S}^*$ is obtained by minimizing the loss function $\mathcal{L}$ on the original dataset $\mathcal{T}$ evaluated using the kernel Ridge regression model trained on $\mathcal{S}$:
    $$\mathcal{S}^* \in \arg \min_{\red{\mathcal{S}}} \mathcal{L} (h^*_{\red{\mathcal{S}}}, \mathcal{T}),$$
    where $h^*_{\mathcal{S}} = K_{\cdot, \mathcal{S}} \big(K_{\mathcal{S}, \mathcal{S}} + \lambda \mathbf{I}_M\big)^{-1} \tilde{\mathbf{y}}_{\mathcal{S}}$ is the closed-form solution to Kernel Ridge regression with the NNGP kernel. This solution provides a compact representation of the original dataset, facilitating efficient compression and processing.
    
    \paragraph{Kernel Ridge regression with Neural Feature Kernel (NFK) \citep{zhou2022dataset}} Let $\mathcal{T}$ denote the original dataset and $\mathcal{S}$ represent the synthetic dataset. We define a kernel function $k$ based on a neural feature mapping, defined by $k (\mathbf{x}_1, \mathbf{x}_2) = h (\mathbf{x}_1)^T h (\mathbf{x}_2)$, where $h$ is a randomly initialized model that maps input data to a higher-dimensional feature space. Then the condensed dataset $\mathcal{S}^*$ is obtained by minimizing the loss function $\mathcal{L}$ with respect to the synthetic dataset, which can be expressed as: 
    $$\mathcal{S}^* \in \arg \min_{\red{\mathcal{S}}} \mathcal{L} (h^*_{\red{\mathcal{S}}}, \mathcal{T}),$$
    where $h^*_{\mathcal{S}} = K_{\cdot, \mathcal{S}} \big(K_{\mathcal{S}, \mathcal{S}} + \lambda \mathbf{I}_M\big)^{-1} \tilde{\mathbf{y}}_{\mathcal{S}}$  is the closed-form solution to Kernel Ridge regression with the NFK kernel. This approach leverages the neural feature mapping to capture complex relationships between data points, enabling effective dataset compression.
    

\subsection{Further Improvements}
To further enhance the quality and effectiveness of the synthetic dataset, we explore the addition of data augmentation techniques and regularization terms to the loss function. The primary goal of these modifications is to improve the downstream performance of models trained on the synthetic dataset, as well as increase its diversity and representativeness compared to the original dataset. 
By incorporating these additional components into the loss function, we aim to create a more robust and generalizable synthetic dataset that can better support a wide range of downstream applications.

\subsubsection{Data augmentation}
The data augmentation techniques discussed in this section are detailed in \Cref{sec:aug}. By augmenting both the original dataset $\mathcal{T}$ and the synthetic dataset $\mathcal{S}$, we can effectively expand the data distribution, thereby enhancing the expressiveness and downstream performance of models trained on the synthetic data. This expansion enables the model to capture a wider range of patterns and relationships, ultimately leading to improved generalization and robustness in downstream applications.
    \paragraph{Distribution matching with Spatial Attention Matching (SAM) \citep{sajedi2023datadam}} We define an attention map $\mathcal{A}: \mathbb{R}^{c \times n} \to \mathbb{R}^{n}$ as $\mathcal{A} (\mathbf{x}) = \sum_{i = 1}^c \|\mathbf{x}_{i, :}\|_p^p$, which operates on input tensors $\mathbf{x}$ of shape $c\times n$. The synthetic dataset $\mathcal{S}^*$ is obtained by solving the following minimax optimization problem:
    $$\mathcal{S}^* \in \arg \min_{\red{\mathcal{S}}} \sup_{h \in \mathcal{H}} \sum_{l = 1}^L \big\|\overline{\mathcal{A} (h_l (\mathcal{T}))} - \overline{\mathcal{A} (h_l (\red{\mathcal{S}}))} \big\|_2^2,$$
    where $h_l$ denotes the intermediate feature maps of a neural network $h$. In practice, we approximate the supremum over the function space $\mathcal{H}$ by averaging the results from a batch of randomly initialized or pre-trained models $h$.
    
    \paragraph{Distribution matching with multi-formation and model enrichment \citep{zhao2023improved}} Let $\mathcal{A}_r$ denote a multi-formation operator that transforms the synthetic dataset. The optimal synthetic dataset $\mathcal{S}^*$ is obtained by minimizing the maximum discrepancy between the expected values of a set of functions applied to the original dataset $\mathcal{T}$ and the transformed synthetic dataset $\mathcal{A}(\mathcal{S})$. This can be formulated as: 
    $$\mathcal{S}^* \in \arg \min_{\red{\mathcal{S}}} \sup_{h \in \mathcal{H}} \big\|\overline{h (\mathcal{T})} - \overline{h (\mathcal{A} (\red{\mathcal{S}}))} \big\|_2^2.$$
    In practice, we approximate the supremum over the function space $\mathcal{H}$ by averaging the outputs of a batch of models, including both randomly initialized and pre-trained models. To further enrich the function space, we also incorporate models trained at intermediate epochs, providing a more comprehensive and diverse set of functions to match the distributions.
    
    \paragraph{Gradient matching with data augmentation \citep{zhao2021dataset-siamese}} Let $\mathcal{A}$ be a differentiable siamese augmentation operator, which applies the same differentiable augmentation to both the original dataset $\mathcal{T}$ and the synthetic dataset $\mathcal{S}$. The optimal synthetic dataset $\mathcal{S}^*$ is obtained by minimizing the maximum discrepancy between the expected gradients of a set of functions applied to the augmented original and synthetic datasets. This can be formulated as: 
    $$\mathcal{S}^* \in \arg \min_{\red{\mathcal{S}}} \sup_{h \in \mathcal{H}} \big\|\overline{\nabla h (\mathcal{A} (\mathcal{T}))} - \overline{\nabla h (\mathcal{A} (\red{\mathcal{S}}))} \big\|_2^2.$$
    In practice, we approximate the supremum over the function space $\mathcal{H}$ by averaging the gradients of a batch of models, including both randomly initialized and pre-trained models $h$. By matching the gradients of the models on the augmented datasets, we aim to capture the underlying patterns and relationships in the data, thereby improving the quality of the synthetic dataset.
    
    \paragraph{Gradient matching with multi-formation \citep{kim2022dataset}} Let $\mathcal{A}_r$ denote a multi-formation operator as described in \Cref{sec:aug}, which applies a set of transformations to the original and synthetic datasets. The optimal synthetic dataset $\mathcal{S}^*$ is obtained by minimizing the maximum discrepancy between the expected gradients of a set of functions applied to the transformed original and synthetic datasets. This can be formulated as: 
    $$\mathcal{S}^* \in \arg \min_{\red{\mathcal{S}}} \sup_{h \in \mathcal{H}} \big\|\overline{\nabla h (\mathcal{A} (\mathcal{T}))} - \overline{\nabla h (\mathcal{A} (\red{\mathcal{S}}))} \big\|_2^2.$$
    In practice, we approximate the supremum over the function space $\mathcal{H}$ by averaging the gradients of a batch of models, including both randomly initialized and pre-trained models $h$. By leveraging the multi-formation operator to generate diverse transformations of the data, we aim to improve the robustness and generalizability of the synthetic dataset. 
    
    \paragraph{Gradient matching with channel-wise multi-formation \citep{zhou2024dataset}} Let $\mathcal{A}_c$ denote a channel-wise multi-formation operator as described in \Cref{sec:aug}, which applies a set of transformations to the original and synthetic datasets in a channel-wise manner. The optimal synthetic dataset $\mathcal{S}^*$ is obtained by minimizing the maximum discrepancy between the expected gradients of a set of functions applied to the transformed original and synthetic datasets. This can be formulated as: 
    $$\mathcal{S}^* \in \arg \min_{\red{\mathcal{S}}} \sup_{h \in \mathcal{H}} \big\|\overline{\nabla h (\mathcal{A} (\mathcal{T}))} - \overline{\nabla h (\mathcal{A} (\red{\mathcal{S}}))} \big\|_2^2.$$
    In practice, we approximate the supremum over the function space $\mathcal{H}$ by averaging the gradients of a batch of models, including both randomly initialized and pre-trained models $h$. By leveraging the channel-wise multi-formation operator to generate diverse transformations of the data, we aim to capture the complex patterns and relationships in the data, thereby improving the quality and robustness of the synthetic dataset. 

    \paragraph{Trajectory matching with neural spectral decomposition \citep{yang2024neural}} Given the collection of spectral tensors $\{\mathbf{T}_i\}$ and kernel tensors $\{\mathbf{K}_j\}$ as described in \Cref{sec:tensor}, we construct a synthetic dataset $\mathcal{S} = \{\mathbf{T}_i \otimes \mathbf{K}_j\}$. To obtain the optimal synthetic dataset $\mathcal{S}^*$, we solve a bi-level optimization problem, which can be formulated as: 
    $$\mathcal{S}^* \in \min_{\red{\mathcal{S}}} \; \|\theta^*_{\red{\mathcal{S}}} - \theta^*_{\mathcal{T}}\|, \; \theta^*_{\red{\mathcal{S}}} \in \min_{\blue{\theta}} \; \mathcal{L} (h_{\blue{\theta}}, \red{\mathcal{S}}).$$
    This bi-level optimization problem is efficiently solved using BPTT, as detailed in \Cref{sec:gd}.

\subsubsection{Loss improvement and regularization}
To further enhance the quality of the synthetic dataset, we can modify the loss objective to incorporate intermediate features of the network, in addition to the output feature. By doing so, we can capture a more comprehensive range of underlying features present in the original dataset, leading to improved loss performance and regularization. This multi-feature loss objective enables the synthetic dataset to better preserve the underlying structure and relationships present in the original data, resulting in a more accurate and informative representation.
    \paragraph{Gradient matching with contrastive signals \citep{lee2022dataset}} The original gradient matching formulation involves computing the distance between the mean gradients of samples from each distinct class and then summing up these distances across classes. This can be expressed as: 
    $$\mathcal{S}^* \in \arg \min_{\red{\mathcal{S}}} \sup_{\blue{h} \in \mathcal{H}} \sum_{y = 1}^C \big\|\overline{\nabla \blue{h} (\mathcal{T}^y)} - \overline{\nabla \blue{h} (\red{\mathcal{S}}^y)} \big\|_2^2,$$
    where $\mathcal{T}^y$ and $\mathcal{S}^y$ denote the subsets of the original and synthetic datasets, respectively, belonging to class y. In practice, the supremum over the function space $\mathcal{H}$ is approximated by iteratively updating the model $h$ after certain iterations of updating the synthetic dataset $\mathcal{S}$.
    
    To improve upon this formulation, \citep{lee2022dataset} propose a modified version of gradient matching that leverages contrastive signals by summing up the mean gradients across classes before computing the difference between the original and synthetic datasets. This can be expressed as: 
    $$\mathcal{S}^* \in \arg \min_{\red{\mathcal{S}}} \bigg\|\sum_{y = 1}^C \overline{\nabla h^*_{\red{\mathcal{S}}} (\mathcal{T}^y)} - \sum_{y = 1}^C \overline{\nabla h^*_{\red{\mathcal{S}}} (\red{\mathcal{S}}^y)} \bigg\|_2^2, \; h^*_{\mathcal{S}} \in \min_{\blue{h} \in \mathcal{H}} \mathcal{L} (\blue{h}, \red{\mathcal{S}}).$$
    The key difference between these two formulations lies in the order of operations: the original formulation computes the difference between mean gradients for each class separately, whereas the improved formulation sums up the mean gradients across classes before computing the difference.

    \paragraph{Gradient matching with loss curvature regularization \citep{shin2023loss}} To incorporate loss-curvature information, we define the Hessian matrix $\mathbf{H}_{\mathcal{T}} (h)$ of the loss function $\mathcal{L} (h, \mathcal{T})$ with respect to the model $h \in \mathcal{H}$. Let $\lambda^{+}_{\mathcal{T}, \mathcal{S}} (h)$ denote the maximum eigenvalue of the matrix $\mathbf{H}_{\mathcal{T}} (h) - \mathbf{H}_{\mathcal{S}} (h)$, which represents the difference in loss curvature between the original dataset $\mathcal{T}$ and the synthetic dataset $\mathcal{S}$. The synthetic dataset $\mathcal{S}^*$ is then obtained by solving the following optimization problem: 
    $$\mathcal{S}^* \in \arg \min_{\red{\mathcal{S}}} \bigg(\sup_{\blue{h} \in \mathcal{H}} \big\|\overline{\nabla \blue{h} (\mathcal{T})} - \overline{\nabla \blue{h} (\red{\mathcal{S}})} \big\|_2^2 + \frac{1}{2} \rho \lambda^{+}_{\mathcal{T}, \red{\mathcal{S}}} (\blue{h}) \bigg),$$
    where $\rho$ is a hyperparameter that controls the strength of the loss-curvature regularization. In practice, the supremum over the function space $\mathcal{H}$ is approximated by iteratively updating the model $h$ after certain iterations of updating the synthetic dataset $\mathcal{S}$.

\subsubsection{Efficiency improvement}
When dealing with large and computationally intractable datasets, a two-stage approach can be employed: first, a rough compression can be achieved using naive and greedy clustering methods, such as K-means, to reduce the dataset size; then, advanced DC methods can be applied to further refine and condense the dataset, resulting in a more compact and representative subset of the original data.
    \paragraph{BPTT with randomized truncation \citep{feng2024embarrassingly}} The standard BPTT algorithm \eqref{eq:dc} requires iterative updates of the model (inner-loop) and the synthetic dataset (outer-loop), resulting in a substantial computational burden. To mitigate this issue, the Truncated BPTT (T-BPTT) method \citep{puskorius1994truncated, williams1990efficient} truncates the outer-loop iteration steps, significantly improving the efficiency of BPTT. Furthermore, the Randomized Truncated BPTT (RaT-BPTT) method \citep{feng2024embarrassingly} introduces an additional layer of randomness to the truncation process, allowing for more flexible and efficient optimization. By incorporating randomness into the truncation, RaT-BPTT can better navigate the complex optimization landscape and converge to a more accurate solution.
    
    \paragraph{Gradient matching with $K$-means clustering \citep{liu2023dream, liu2023dream+}} To reduce the computational complexity of DC methods, we apply a clustering operator $cl$ to the original dataset $\mathcal{T}$ every $T$ steps, resulting in a subset of $\mathcal{T}$ with a fixed size. This subset serves as a proxy for $\mathcal{T}$, allowing us to efficiently conduct DC methods. The synthetic dataset $\mathcal{S}^*$ is then obtained by solving the following optimization problem: 
    $$\mathcal{S}^* \in \arg \min_{\red{\mathcal{S}}} \sup_{\blue{h} \in \mathcal{H}} \big\|\overline{\nabla \blue{h} (cl (\mathcal{T}))} - \overline{\nabla \blue{h} (\red{\mathcal{S}})} \big\|_2^2.$$
    In practice, we approximate the supremum over the function space $\mathcal{H}$ by iteratively updating the model $h$ after a certain number of iterations of updating the synthetic dataset $\mathcal{S}$. This approach enables efficient gradient matching with a reduced computational burden, while still maintaining the accuracy of the DC methods.

Using the notions of discrepancy, spaces and additional tricks presented, we present a high-level taxonomy of eminent DC methods in \Cref{tab:overview}.

\input{taxonomy_table_revision}
\footnotetext{While NTK/NNGP-based methods can be interpreted as shaping the outputs of the kernel-induced predictor, their primary contribution is not to match features or outputs explicitly. Instead, they reformulate the dataset condensation objective in a linearized regime where training admits a closed-form solution, thereby converting the original bi-level optimization into a single-level problem.}

\section{Multi-Objective-Aware Dataset Compression}
The primary objective of DC is to distill large datasets into smaller, more efficient representations while preserving downstream performance, i.e., accuracy. By doing so, DC aims to address the long-standing challenge of jointly optimizing efficiency and accuracy in model training. However, in recent years, it has become increasingly evident that accuracy and efficiency are no longer sufficient metrics for evaluating the quality of a dataset. Concerns have been raised regarding the social aspects of models, such as privacy and robustness, and how they are affected when training models on condensed synthetic datasets. Specifically, it remains unclear how DC impacts the privacy and robustness of models, and whether it is possible to jointly improve accuracy, privacy, and robustness through DC. This poses a significant challenge and open question in the field, highlighting the need for further research into the broader implications of DC on model training and deployment.

\rev{This section is an \emph{organizing discussion} of prior work; its only new formal object is the following problem statement, which fixes what ``multi-objective DC'' means in our framework. Given objectives $J_{\mathrm{acc}}$, $J_{\mathrm{priv}}$, and $J_{\mathrm{rob}}$, e.g., the transfer gap $\mathrm{GD}_{\mathcal{A}}$, a privacy loss such as the differential-privacy parameter $\varepsilon_{\mathrm{DP}}$ delivered by the synthesis mechanism, and an adversarial risk, multi-objective DC is either the constrained problem
$$\inf_{\red{\nu} \in \mathcal{Q}_M} \; J_{\mathrm{acc}} (\red{\nu}) \quad \text{s.t.} \quad J_{\mathrm{priv}} (\red{\nu}) \le \varepsilon_{\mathrm{p}}, \quad J_{\mathrm{rob}} (\red{\nu}) \le \varepsilon_{\mathrm{r}},$$
or a scalarization $\inf_{\nu \in \mathcal{Q}_M} J_{\mathrm{acc}} (\nu) + \lambda_{\mathrm{p}} J_{\mathrm{priv}} (\nu) + \lambda_{\mathrm{r}} J_{\mathrm{rob}} (\nu)$ with $\lambda_{\mathrm{p}}, \lambda_{\mathrm{r}} \ge 0$. We make no claims about the shape of the attainable trade-off surface; characterizing it is an open problem (see the concluding discussion).}

\subsection{Privacy-aware compression}
In this context, we focus on differential privacy (DP) as the primary notion of privacy. Specifically, a model $h \in \mathcal{H}$ is considered $(\varepsilon, \delta)$-differentially private if it satisfies the following condition: for any possible set of model outputs $\Lambda$ and any pair of datasets $\mathcal{T}$ and $\mathcal{T}'$ that differ by only one element, the probability of $h(\mathcal{T})$ being in $\Lambda$ is bounded by $e^{\varepsilon} \cdot \mathbf{Pr} [h(\mathcal{T}') \in \Lambda] + \delta$. This can be formally expressed as: 
$$\mathbf{Pr} [h (\mathcal{T}) \in \Lambda] \le e^{\varepsilon} \cdot \mathbf{Pr} [h (\mathcal{T}') \in \Lambda] + \delta.$$
By imposing DP constraints on the model, we can establish an upper bound on the amount of information that can be leaked by any individual sample in the dataset. The primary objective of DP-aware dataset condensation is to embed DP into the synthetic datasets, resulting in models that possess enhanced DP properties. It is worth noting that other notions of privacy, such as membership inference privacy (MIP) \citep{shokri2017membership}, have also been explored. For instance, \citep{dong2022privacy} have demonstrated that synthetic datasets obtained via distribution matching can provide a certain level of MIP. However, the development of DP-aware dataset condensation methods remains a crucial area of research, as it has the potential to provide strong privacy guarantees for models trained on synthetic datasets.

    \paragraph{Distribution matching using MMD for DP \citep{harder2021dpmerf}} \rev{To achieve differential privacy, DP-MERF privatizes a \emph{bounded random-feature mean embedding} of the data, not the raw data. Let $\phi (\mathbf{x}) \in \mathbb{R}^p$ be random Fourier features normalized so that $\|\phi (\mathbf{x})\|_2 \le c$ uniformly, and let $\tilde{\mu}_{\mathcal{T}} := \frac{1}{N} \sum_{i} \phi (\mathbf{x}_i)$ be the empirical mean embedding. Replacing one record changes $\tilde{\mu}_{\mathcal{T}}$ by at most $2 c / N$ in $\ell_2$-norm, so the Gaussian mechanism releases
    $$\tilde{\mu}_{\mathcal{T}}^{\mathrm{DP}} := \tilde{\mu}_{\mathcal{T}} + \boldsymbol{\varepsilon}, \qquad \boldsymbol{\varepsilon} \sim \mathcal{N} (0, \sigma^2 \mathbf{I}_p),$$
    with $\sigma$ calibrated to the sensitivity $2 c / N$ and the target $(\varepsilon, \delta)$-DP budget. The synthetic dataset is then obtained by matching against the privatized embedding:
    $$\mathcal{S}^* \in \arg \min_{\red{\mathcal{S}}} \Big\|\tilde{\mu}_{\mathcal{T}}^{\mathrm{DP}} - \frac{1}{M} \sum_{j} \phi (\tilde{\mathbf{x}}_j)\Big\|_2^2,$$
    which is a (privatized) empirical squared MMD with the random-feature kernel. The DP guarantee for $\mathcal{S}^*$ and for anything subsequently trained on it follows from the Gaussian mechanism applied to the bounded embedding together with the \emph{post-processing} property of DP. Since the private statistic is computed once, the privacy budget does not degrade with the number of synthesis iterations.}
    
    \paragraph{Gradient matching for private set generation \citep{chen2022private}} To generate private synthetic datasets, \rev{per-sample gradients on the original dataset are clipped to bound their sensitivity, and sensitivity-calibrated Gaussian noise is added to their average (the Gaussian mechanism); the DP guarantee for the synthetic dataset then follows from composition over iterations together with the post-processing property}. The privacy-aware synthetic dataset $\mathcal{S}^*$ is obtained by solving the following optimization problem: 
    $$\mathcal{S}^* \in \arg \min_{\red{\mathcal{S}}} \sup_{\blue{h} \in \mathcal{H}} \big\|\big(\overline{\nabla \blue{h} (\mathcal{T})} + \epsilon \big) - \overline{\nabla \blue{h} (\red{\mathcal{S}})} \big\|_2^2,$$
    where \rev{$\epsilon$ is the calibrated Gaussian noise vector, drawn from $\mathcal{N} (0, \Sigma)$ with $\Sigma$ determined by the clipping norm and the privacy budget}. In practice, the supremum over the function space $\mathcal{H}$ is approximated by iteratively updating the model $h$ after a certain number of iterations of updating the synthetic dataset $\mathcal{S}$. By adding \rev{calibrated} noise to the \rev{clipped} average gradients, this approach ensures that the synthetic dataset $\mathcal{S}^*$ satisfies the DP guarantees, while the gradient matching process enables the generation of high-quality synthetic datasets that preserve the underlying patterns and relationships of the original data.

\subsection{Robustness-aware compression}
In this context, we focus on adversarial robustness as the primary notion of robustness. Given a perturbation radius $\varepsilon \ge 0$, a norm $\|\cdot\|$, and labeled data $(\mathbf{x}, y)$ sampled from a probability distribution $\mu$, the $\varepsilon$-adversarial robustness of a model $h \in \mathcal{H}$ is quantified by the probability that the model correctly classifies the input data even when it is perturbed by an adversarial attack within the specified radius: $\mathbf{Pr} [h (\mathbf{x} + \delta) = y, \; \forall \|\delta\| \le \varepsilon]$. A common approach to improving the adversarial robustness of a model is to minimize the adversarial loss function:
$$\mathcal{L}^{adv} (h, \mu; \varepsilon) := \mathbb{E}_{(\mathbf{x}, y) \sim \mu} \bigg[\max_{\|\red{\delta}\| \le \varepsilon} l (h(\mathbf{x} + \red{\delta}), y)\bigg].$$
Recent work by \citep{ilyas2019adversarial} has shown that it is possible to extract robust and non-robust features from a dataset, such that standard model training over robust (or non-robust) features leads to robust (or non-robust) models. This raises an intriguing question: {\em Can we further enhance the condensed dataset by embedding adversarial robustness into it}? In other words, can we design a dataset condensation method that not only preserves the key features of the original dataset but also improves the robustness of models trained on the condensed dataset?

    \paragraph{Adversarially robust condensation by bi-level optimization \citep{wu2022towards}} To obtain a synthetic dataset with robustness guarantees, similar to formulation \eqref{eq:dc} for dataset condensation, we can employ the following bi-level optimization objective:
    \begin{align*} \label{eq:RobDC}
        \mu_2^* \in \inf_{\red{\mu_2}} \; \mathcal{L}^{adv} (h^*_{\red{\mu_2}}, \mu_1; \varepsilon), \; \text{s.t. } h^*_{\red{\mu_2}} \in \arg\min_{\blue{h}} \; \mathcal{L} (\blue{h}, \red{\mu_2}), \tag{RobDC}
    \end{align*}
    where $\mu_1$ represents the original distribution and $\mu_2$ denotes the synthetic distribution. The formulation in \eqref{eq:RobDC} reveals that the model trained on the optimized synthetic distribution $\mu_2^*$ achieves optimal robustness performance on the original dataset. By solving this bi-level optimization problem, we can generate a synthetic dataset that not only preserves the key characteristics of the original data but also enhances the robustness of models trained on it. This approach provides a promising way to improve the adversarial robustness of machine learning models while maintaining their performance on the original task.
    
    \paragraph{Adversarially robust condensation by kernel Ridge regression \citep{tsilivis2022can}} Let $h^*_{\mathcal{S}} = K_{\cdot, \mathcal{S}} \big(K_{\mathcal{S}, \mathcal{S}} + \lambda \mathbf{I}_M\big)^{-1} \tilde{\mathbf{y}}_{\mathcal{S}}$ denote the closed-form solution of kernel Ridge regression trained on the synthetic dataset $\mathcal{S}$. The method jointly optimizes the synthetic dataset $\mathcal{S}$ and an adversarial perturbation $\delta$ applied to the original dataset $\mathcal{T}$:
    \begin{align*} \label{eq:RidgeDC}
        (\mathcal{S}^*, \delta^*) \in \arg \min_{\red{\mathcal{S}}} \max_{\blue{\delta}} \mathcal{L} (h^*_{\red{\mathcal{S}}}, \mathcal{T} + \blue{\delta}). \tag{RidgeDC}
    \end{align*}
    This formulation defines a two-player game: the adversarial perturbation $\delta$ aims to degrade the performance of the model $h^*_{\mathcal{S}}$, while the synthetic dataset $\mathcal{S}$ is optimized to ensure robustness and strong generalization performance of $h^*_{\mathcal{S}}$ even on the perturbed dataset. Compared to the adversarial objective in \citep{wu2022towards}, which applies perturbations at the level of individual data points and maximizes each per-instance loss before averaging, this method performs the maximization over the regularized loss on the entire perturbed dataset. This shift from pointwise to dataset-level robustness yields a more global form of adversarial resistance during condensation.
    
    \paragraph{Adversarially robust condensation by curvature regularization \citep{xue2025towards}} Motivated by the established link between adversarial robustness and model smoothness, which is quantified through curvature and the Lipschitz constant, this method augments the standard dataset condensation objective with a curvature regularization term that encourages locally flat loss landscapes:
    \begin{align*} \label{eq:CurvDC}
        \mu_2^* \in \inf_{\red{\mu_2}} \; \mathcal{L} (h^*_{\red{\mu_2}}, \mu_1) + \lambda L_*, \; \text{s.t. } h^*_{\red{\mu_2}} \in \arg\min_{\blue{h}} \; \mathcal{L} (\blue{h}, \red{\mu_2}), \tag{CurvDC}
    \end{align*}
    Here, $\mu_1$ is the original data distribution, $\mu_2$ is the learned synthetic data distribution, and $L_*$ is a regularization term estimating the largest eigenvalue of the Hessian of the loss function $l$, serving as a proxy for curvature. By training models on the synthetic distribution $\mu_2$ while explicitly penalizing high curvature, this approach biases learning toward flatter minima, i.e., regions of the loss surface with low sensitivity to input perturbations. Consequently, the resulting models exhibit reduced Lipschitz constants and enhanced adversarial robustness, even though they are trained on a compressed synthetic dataset. \rev{\emph{Scope of the guarantee.} The supporting analysis provides a \emph{conditional} upper bound on a quadratic (curvature-based) approximation to the adversarial loss, under assumptions that include local convexity of the loss in the relevant region and Lipschitz-type conditions; the reported robustness gains are measured empirically against standard attacks. We stress that input-loss curvature, parameter-space curvature, input Lipschitz constants, empirical attack robustness, and \emph{certified} robustness are distinct notions, and none of the results discussed in this subsection constitutes a robustness certificate.}


\rev{\section{Empirical Illustration: Does the Hierarchy Track Transfer?} \label{sec:exp}
The hierarchy of \Cref{sec:dc} asserts \emph{sufficiency}: under the stated assumptions, small surrogate discrepancies imply small transfer gaps. It does not assert a performance ordering of method families, and Example~\ref{ex:zero_gd} shows why no such ordering should be expected. To make the sufficiency reading concrete, we accompany the framework with a small controlled study.

\paragraph{Protocol.} CIFAR-10 with per-class budgets $\mathrm{IPC} \in \{1, 10, 50\}$; one representative method per discrepancy level: RaT-BPTT \citep{feng2024embarrassingly} for GD, feature matching (CAFE) \citep{wang2022cafe} for the VD surrogate, trajectory matching \citep{cazenavette2022dataset} for PD, and distribution matching \citep{zhao2023dataset} together with M3D \citep{zhang2024m3d} for DD/MMD; a fixed ConvNet backbone; a shared differentiable-augmentation and evaluation pipeline; and $5$ random seeds with mean $\pm$ standard deviation. For every condensed set we report (i) the test accuracy of a model retrained from scratch, (ii) the condensation compute in GPU-hours (under our fixed single-GPU setup, the corresponding CO$_2$e estimate is linearly proportional: approximately $0.063$\,kg per GPU-hour, assuming a $300$\,W TDP, PUE $1.4$, and grid intensity $0.15$\,kg\,CO$_2$e/kWh, and is therefore not tabulated separately), and (iii) the \emph{measured} discrepancies, i.e., the empirical MMD with the induced random-feature kernel, the coupled parameter distance along matched trajectories, and the empirical transfer gap, that enables a direct check of whether small measured surrogates entail small transfer gaps across methods and budgets, as \Cref{prop:transfer,prop:bridge} predict under their assumptions.

}

\begin{figure}[t]
    \centering
    \includegraphics[width=0.6\linewidth]{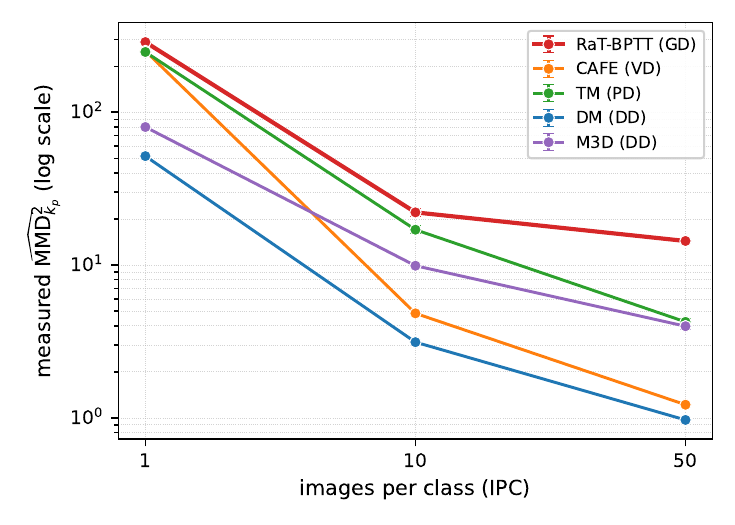}
    \caption{\rev{Measured induced-kernel MMD $\widehat{\mathrm{MMD}}^2_{k_p}$ (log scale) by per-class budget, for the study of \Cref{sec:exp}; error bars show $\pm 1$ standard deviation over 5 seeds (full per-method results in Table~\ref{tab:exp}, Appendix~\ref{sec:app_exp}). Within every method, the measured discrepancy falls monotonically as the budget grows, which aligns with the sufficiency direction of \Cref{prop:transfer,prop:bridge}. Across methods, RaT-BPTT, which optimizes the task-level objective directly, retains the \emph{largest} measured MMD at every budget while attaining the smallest or near-smallest transfer gap (Table~\ref{tab:exp}), whereas the DD-type methods (DM, M3D) attain the smallest values: distribution matching is sufficient but not necessary (Example~\ref{ex:zero_gd}).}}
    \label{fig:exp}
\end{figure}

\rev{\paragraph{What the study can and cannot test.} The hierarchy of \Cref{sec:dc} is a chain of upper bounds, proven under the stated assumptions, whose constants the theory defines but does not determine numerically for practical settings: $L$ is the Lipschitz constant of the loss (computable for standard losses such as cross-entropy on logits), whereas $C$ is a parameter-Lipschitz constant of the architecture over the symmetry orbit (assumption (b) of \Cref{sec:dc}), for which available bounds on deep networks are loose to the point of vacuity; the transfer bound additionally carries the optimization term $2\varepsilon$ of \Cref{prop:transfer}. These constants are absorbed into the arrows of the schematic. An experiment therefore cannot \emph{verify} the inequalities, since they hold by proof. And because the end-to-end constant $LC$ is unknown in practice, it cannot falsify them by measuring realized ratios either. What an experiment can meaningfully assess is their \emph{practical informativeness}: whether the measurable surrogate discrepancies co-vary with, and order, the observed transfer gap despite the unknown constants. Accordingly, we report the transfer gap together with the two measurable ends of the chain: a DD instance ($\widehat{\mathrm{MMD}}^2_{k_p}$) and the symmetry-aligned parameter distance $\mathrm{PD}_{\mathcal{G}}$, the latter under the coupled protocol that the hierarchy's assumption (c) requires; the intermediate value discrepancy is subject to the same constant-identification caveat and is not separately tabulated.}

\rev{\paragraph{Results.} Figure~\ref{fig:exp} plots the measured MMD by budget; the full per-method results (accuracy, transfer gap, $\mathrm{PD}_{\mathcal{G}}$, and compute) are given in Table~\ref{tab:exp} in Appendix~\ref{sec:app_exp}. Two patterns stand out, and both are the ones the theory predicts. First, \emph{within} each method family, the measured surrogate discrepancies and the transfer gap move together: as the budget grows, $\widehat{\mathrm{MMD}}^2_{k_p}$ drops by up to two orders of magnitude, $\mathrm{PD}_{\mathcal{G}}$ decreases, and the gap shrinks monotonically for all five methods, which is consistent with the sufficiency direction of \Cref{prop:transfer,prop:bridge}; no small-discrepancy/large-gap case occurs anywhere in the table. Second, \emph{across} families at a fixed budget, a smaller measured distribution discrepancy does not order accuracy: at $\mathrm{IPC} = 10$, RaT-BPTT attains the smallest transfer gap ($29.8\%$) with an induced-kernel MMD roughly seven times larger than DM's, and at $\mathrm{IPC} = 50$ its gap is within one point of DM's while its measured MMD is fifteen times larger, which follows the empirical counterpart of the non-necessity discussed in Example~\ref{ex:zero_gd}, since RaT-BPTT optimizes the task-level objective directly rather than a distribution discrepancy. Interestingly, the coupled, permutation-aligned parameter distance $\mathrm{PD}_{\mathcal{G}}$ orders the methods consistently with the transfer gap at every budget (up to ties and one adjacent swap within error bars), including the RaT-BPTT rows where the MMD does not, which is an empirical manifestation of the $\mathrm{GD}_{\mathcal{A}} \le LC \cdot \mathrm{PD}_{\mathcal{G}}$ arrow of \Cref{sec:dc}, measured on training runs separate from the evaluation runs. The compute column quantifies the price of bi-level objectives: trajectory matching (TM) at $\mathrm{IPC} = 50$ costs roughly three times DM's condensation compute (excluding its expert-buffer construction, amortized across runs) for lower accuracy in this shared pipeline. Absolute accuracies trail the best published per-method numbers, as expected under a single shared backbone, augmentation, and evaluation pipeline without per-method tuning; the comparisons of interest here are internal to the controlled setup.}

\section{Conclusion and Discussion} \label{sec:conclusion}
In this work, we introduced a unified theoretical framework for DC, aiming to provide a principled foundation for a field that has so far developed in a largely heuristic and fragmented manner. By formalizing DC through the lens of distribution discrepancies and data/hypothesis spaces, we established a general definition that situates existing methods within a common structure. This allowed us to categorize the landscape of DC approaches along several key dimensions: the discrepancy metric employed, the space in which distributions are matched and optimized, the parameterization of models, and the choice of loss functions and regularization.

Our framework not only clarifies the assumptions and limitations of current approaches but also highlights how DC can be extended beyond its conventional goal of preserving generalization accuracy. In particular, we discussed how DC can be viewed as a multi-objective problem, where efficiency, robustness, and privacy become central considerations alongside accuracy. This perspective opens the door to more holistic objectives in the design of synthetic datasets and learning algorithms.

\rev{\paragraph{Open challenges.} The unification exercise itself exposed concrete gaps that we state as open problems. (i) Several widely used objectives are estimators, relaxations, or analogies rather than instances of the discrepancies that motivate them (\Cref{tab:overview}); upgrading the remaining analogies to assumption-aware guarantees is open. (ii) Bi-level formulations carry a substantial compute cost, and the emergence of uni-level, decoupled formulations (\Cref{sec:unilevel}) shifts the theoretical question toward what model-specific statistic matching provably preserves. (iii) Storage accounting is inconsistent across the literature: pixels, soft labels, and generator or decoder parameters are budgeted differently, which confounds comparisons under a nominal budget $M$. (iv) Cross-architecture transfer falls outside single-$\mathcal{H}$ guarantees, and the hierarchy constants ($L$, $C$) are architecture-dependent. (v) For multi-objective DC, the attainable accuracy--privacy--robustness surface is uncharacterized.}

\rev{\paragraph{Benchmark guidelines.} A systematic benchmark remains future work, but our framework prescribes its design: equalized storage accounting (images, labels, and any generator or decoder parameters counted against the same budget); multiple IPC budgets; fixed training and evaluation protocols with a cross-architecture evaluation grid; multiple seeds with variance reporting; reporting of condensation compute (GPU-hours and estimated CO$_2$e via carbon-tracking tools \citep{anthony2020carbontracker}), consonant with the efficiency motivation of \Cref{sec:intro}; and, for any multi-objective claim, separate tracks for privacy, such as formal DP accounting which is clearly distinguished from empirical membership-inference evaluation, and for robustness under standardized attacks.}

Looking forward, the theoretical contributions presented here call for further empirical investigation. A systematic benchmark grounded in our framework\rev{, along the lines above,} would provide a much-needed basis for comparing methods on equal footing, assessing their strengths and weaknesses across multiple objectives, and guiding the development of new algorithms. Establishing such a benchmark represents an important next step toward bridging theory and practice in dataset condensation, ultimately advancing the reliability and applicability of DC in real-world settings.

\rev{\subsubsection*{Broader Impact Statement}
Two practical caveats deserve emphasis. \emph{Compute footprint:} dataset condensation is itself compute-intensive (bi-level methods in particular), while its efficiency gains are realized only downstream; the net footprint is positive only when the condensation cost is amortized over sufficiently many subsequent trainings. We therefore advocate reporting condensation compute alongside accuracy (see the benchmark guidelines above and \Cref{sec:exp}). \emph{Privacy and robustness:} formal differential-privacy guarantees require sensitivity-calibrated mechanisms combined with post-processing arguments, as in the corrected description of \citet{harder2021dpmerf}; empirical privacy observations (e.g., resistance to membership-inference attacks) must not be conflated with such guarantees, and nothing in this paper should be read as claiming that condensation per se provides provable privacy. Likewise, surrogate curvature bounds must not be read as certified robustness.}

%% file: taxonomy_table_revision.tex
\begin{table}
    \providecommand{\revcell}[1]{\textcolor{black}{#1}}
    \caption{Overview of the taxonomy for dataset condensation methods. Legend for the symbols in Tricks column: (data$^\dagger$, model$^\ddagger$, augmentation$^*$, regularization$^\mathparagraph$, efficiency$^\mathsection$, multi-objective$^\intercal$). \revcell{The Correspondence column states the type of relationship between each method's practical objective and the formal discrepancy of its row (exact, possibly with an induced kernel; estimator; upper bound; relaxation; or analogy), as established in Sections~3 and~5.}}
    \label{tab:overview}
    \vskip .1in
    \centering
    \tiny
    \setlength{\tabcolsep}{1.05em} 
    \renewcommand{\arraystretch}{1.2} 
    \rotatebox{0}{
        \begin{tabular}{|c|p{13.8em}| p{13em} |c| p{14em} |}
            \cline{1-5}
            \multirow{2}{*}{\bf Discrepancy} & \multirow{2}{*}{\centering{\bf Description}} & \multirow{2}{*}{{\bf Tricks}} & \multirow{2}{*}{{\bf Method}} & \multirow{2}{*}{\revcell{{\bf Correspondence}}} \\
            &&&& \\
            \cline{1-5}
            \multirow{10}{*}{$\mathrm{GD}$} & \multirow{2}{*}{BPTT for gradients} & \multirow{2}{*}{N/A} & \multirow{2}{*}{\cite{wang2018dataset}} & \multirow{2}{*}{\revcell{exact (empirical $\mathrm{GD}_{\mathcal{A}}$)}} \\
            &&&& \\
            \cline{2-5}
            & \multirow{2}{*}{BPTT for gradients}  & \multirow{2}{*}{randomized truncation$^\mathsection$} & \multirow{2}{*}{\cite{feng2024embarrassingly}} & \multirow{2}{*}{\revcell{estimator (truncated $\mathrm{GD}_{\mathcal{A}}$)}} \\
            &&&& \\
            \cline{2-5}
            & \multirow{2}{*}{CIG for gradients}  & \multirow{2}{*}{N/A} & \multirow{2}{*}{\cite{loo2023dataset}} & \multirow{2}{*}{\revcell{estimator}} \\
            &&&& \\
            \cline{2-5}
            & \multirow{2}{*}{N/A} & \multirow{2}{*}{robustness$^\intercal$} & \multirow{2}{*}{\cite{wu2022towards}} & \multirow{2}{*}{\revcell{exact (adversarial $\mathrm{GD}_{\mathcal{A}}$)}} \\
            &&&& \\
            \cline{2-5}
            & \multirow{2}{*}{N/A} & curvature regularization$^\mathparagraph$, robustness$^\intercal$ & \multirow{2}{*}{\cite{xue2025towards}} & \multirow{2}{*}{\revcell{relaxation (curvature surrogate)}} \\
            \cline{1-5}
            \multirow{15}{*}{$\mathrm{VD}$} & \multirow{7}{*}{intermediate feature matching} & \multirow{2}{*}{with discrimination loss$^\mathparagraph$} & \multirow{2}{*}{\cite{wang2022cafe}} & \multirow{2}{*}{\revcell{analogy (statistic surrogate)}} \\
            &&&& \\
            \cline{3-5}
            && GAN$^\dagger$, intra-class diversity$^\mathparagraph$, inter-class discrimination$^\mathparagraph$ & \multirow{2}{*}{\cite{zhang2023dataset}} & \multirow{2}{*}{\revcell{analogy (statistic surrogate)}} \\
            \cline{3-5}
            && diffusion model training$^\dagger$, representative loss$^\mathparagraph$, diversity loss$^\mathparagraph$ & \multirow{3}{*}{\cite{gu2024efficient}} & \multirow{3}{*}{\revcell{analogy (statistic surrogate)}} \\
            \cline{2-5}
            & \multirow{8}{*}{linearized training objective \footnotemark} & Kernel Ridge regression with NTK$^\ddagger$ & \multirow{2}{*}{\cite{nguyen2021dataset-b, nguyen2021dataset-a}} & \multirow{2}{*}{\revcell{exact (linearized regime)}} \\
            \cline{3-5}
            && Kernel Ridge regression with NNGP$^\ddagger$ & \multirow{2}{*}{\cite{loo2022efficient}} & \multirow{2}{*}{\revcell{exact (linearized regime)}} \\
            \cline{3-5}
            && Kernel Ridge regression with NFK$^\ddagger$ & \multirow{2}{*}{\cite{zhou2022dataset}} & \multirow{2}{*}{\revcell{exact (linearized regime)}} \\
            \cline{3-5}
            && Kernel Ridge regression with NTK$^\ddagger$, robustness$^\intercal$ & \multirow{2}{*}{\cite{tsilivis2022can}} & \multirow{2}{*}{\revcell{exact (linearized regime)}} \\
            \cline{1-5}
            \multirow{4}{*}{$\mathrm{PD}$} & \multirow{4}{*}{trajectory matching} & \multirow{2}{*}{N/A} & \multirow{2}{*}{\cite{cazenavette2022dataset}} & \multirow{2}{*}{\revcell{relaxation (coupled aligned $\mathrm{PD}_{\mathcal{G}}$)}} \\
            &&&& \\
            \cline{3-5}
            && \multirow{2}{*}{neural spectral decomposition$^*$} & \multirow{2}{*}{\cite{yang2024neural}} & \multirow{2}{*}{\revcell{relaxation (coupled aligned $\mathrm{PD}_{\mathcal{G}}$)}} \\
            &&&& \\
            \cline{1-5}
            \multirow{2}{*}{$\mathrm{CD}$} & \multirow{2}{*}{neural characteristic function} & \multirow{2}{*}{N/A} & \multirow{2}{*}{\citep{wang2025dataset}} & \revcell{exact (MMD, learned spectral measure)} \\
            \cline{1-5}
            \multirow{40}{*}{$\mathrm{DD}$} & \multirow{12}{*}{distribution matching (IPM)} & \multirow{2}{*}{N/A} & \multirow{2}{*}{\cite{zhao2023dataset}} & \multirow{2}{*}{\revcell{exact (MMD, induced kernel $k_p$)}} \\
            &&&& \\
            \cline{3-5}
            && \multirow{2}{*}{pre-trained GANs$^\dagger$} & \multirow{2}{*}{\cite{zhao2022synthesizing}} & \multirow{2}{*}{\revcell{exact (MMD, induced kernel $k_p$)}} \\
            &&&& \\
            \cline{3-5}
            && \multirow{2}{*}{GAN training$^\dagger$} & \multirow{2}{*}{\cite{wang2025dim}} & \multirow{2}{*}{\revcell{exact (MMD, induced kernel $k_p$)}} \\
            &&&& \\
            \cline{3-5}
            && \multirow{2}{*}{expert subspace projection$^\ddagger$} & \multirow{2}{*}{\cite{ma2023dataset}} & \multirow{2}{*}{\revcell{exact (MMD, induced kernel $k_p$)}} \\
            &&&& \\
            \cline{3-5}
            && \multirow{2}{*}{spatial attention matching$^*$} & \multirow{2}{*}{\cite{sajedi2023datadam}} & \multirow{2}{*}{\revcell{exact (MMD, induced kernel $k_p$)}} \\
            &&&& \\
            \cline{3-5}
            && \multirow{2}{*}{multi-formation$^*$} & \multirow{2}{*}{\cite{zhao2023improved}} & \multirow{2}{*}{\revcell{exact (MMD, induced kernel $k_p$)}} \\
            &&&& \\
            \cline{2-5}
            & \multirow{16}{*}{gradient matching (IPM)} & \multirow{2}{*}{N/A} & \multirow{2}{*}{\cite{zhao2021dataset}} & \revcell{exact in expectation (one-step) / heuristic (iterated)} \\
            \cline{3-5}
            && \multirow{2}{*}{siamese augmentation$^*$} & \multirow{2}{*}{\cite{zhao2021dataset-siamese}} & \multirow{2}{*}{\revcell{heuristic estimator (iterated)}} \\
            &&&& \\
            \cline{3-5}
            && \multirow{2}{*}{multi-formation$^*$} & \multirow{2}{*}{\cite{kim2022dataset}} & \multirow{2}{*}{\revcell{heuristic estimator (iterated)}} \\
            &&&& \\
            \cline{3-5}
            && \multirow{2}{*}{channel-wise multi-formation$^*$} & \multirow{2}{*}{\cite{zhou2024dataset}} & \multirow{2}{*}{\revcell{heuristic estimator (iterated)}} \\
            &&&& \\
            \cline{3-5}
            && \multirow{2}{*}{contrastive signal$^\mathparagraph$} & \multirow{2}{*}{\cite{lee2022dataset}} & \multirow{2}{*}{\revcell{heuristic estimator (iterated)}} \\
            &&&& \\
            \cline{3-5}
            && \multirow{2}{*}{loss curvature regularization$^\mathparagraph$} & \multirow{2}{*}{\cite{shin2023loss}} & \multirow{2}{*}{\revcell{heuristic estimator (iterated)}} \\
            &&&& \\
            \cline{3-5}
            && \multirow{2}{*}{private set generation$^\intercal$} & \multirow{2}{*}{\cite{chen2022private}} & \revcell{estimator $+$ DP (Gaussian mechanism)} \\
            \cline{3-5}
            && \multirow{2}{*}{K-means clustering$^\mathsection$} & \multirow{2}{*}{\cite{liu2023dream, liu2023dream+}} & \multirow{2}{*}{\revcell{heuristic estimator (iterated)}} \\
            &&&& \\
            \cline{2-5}
            & higher-order moment matching (IPM) & \multirow{2}{*}{N/A} & \multirow{2}{*}{\cite{yu2025teddy}} & \revcell{relaxation (lifted moment matching)} \\
            \cline{2-5}
            & \multirow{2}{*}{mean feature matching (IPM)} & \multirow{2}{*}{K-means clustering$^\mathsection$} & \multirow{2}{*}{\cite{su2024d4m}} & \multirow{2}{*}{\revcell{analogy (clustered features)}} \\
            &&&& \\
            \cline{2-5}
            & \multirow{4}{*}{MMD with NTK} & \multirow{2}{*}{N/A} & \multirow{2}{*}{\cite{zhang2024m3d}} & \multirow{2}{*}{\revcell{exact (MMD)}} \\
            &&&& \\
            \cline{3-5}
            && \multirow{2}{*}{differential privacy$^\intercal$} & \multirow{2}{*}{\cite{harder2021dpmerf}} & \multirow{2}{*}{\revcell{exact (MMD, RFF) $+$ DP}} \\
            &&&& \\
            \cline{2-5}
            & \multirow{2}{*}{Wasserstein metric (WM)} & \multirow{2}{*}{N/A} & \multirow{2}{*}{\cite{liu2023dataset}} & \multirow{2}{*}{\revcell{relaxation (of $\mathrm{WM}_1$ dual)}} \\
            &&&& \\
            \cline{2-5}
            & \multirow{2}{*}{Hausdorff distance} & \multirow{2}{*}{N/A} & \multirow{2}{*}{\cite{chen2024adversarial}} & \multirow{2}{*}{\revcell{analogy (support-level)}} \\
            &&&& \\
            \cline{1-5}
            \multirow{6}{*}{General} & \multirow{6}{*}{N/A} & \multirow{2}{*}{pre-trained GANs$^\dagger$} & \multirow{2}{*}{\cite{cazenavette2023generalizing}} & \multirow{2}{*}{\revcell{exact (generator-constrained $\mathcal{Q}$)}} \\
            &&&& \\
            \cline{3-5}
            && \multirow{2}{*}{pre-trained diffusion models$^\dagger$} & \multirow{2}{*}{\cite{moser2024latent}} & \multirow{2}{*}{\revcell{exact (generator-constrained $\mathcal{Q}$)}} \\
            &&&& \\
            \cline{3-5}
            && \multirow{2}{*}{pre-trained hallucinators$^\dagger$} & \multirow{2}{*}{\cite{liu2022dataset}} & \multirow{2}{*}{\revcell{exact (generator-constrained $\mathcal{Q}$)}} \\
            &&&& \\
            \cline{1-5}
            \multirow{4}{*}{\revcell{Uni-level}} & \multirow{2}{*}{\revcell{squeeze / recover / relabel}} & \revcell{BN statistic matching$^\ddagger$, soft relabeling$^\dagger$} & \multirow{2}{*}{\revcell{\cite{yin2023squeeze}}} & \revcell{exact (model-specific statistics, single-level)} \\
            \cline{2-5}
            & \multirow{2}{*}{\revcell{patch selection $+$ relabel}} & \revcell{realism/diversity scoring$^\mathsection$, soft relabeling$^\dagger$} & \multirow{2}{*}{\revcell{\cite{sun2024diversity}}} & \revcell{exact (selection $+$ relabel, single-level)} \\
            \cline{1-5}
        \end{tabular}
    }
\end{table}

%% file: acks.tex
TC and RS acknowledge funding received under European Union’s Horizon Europe Research and Innovation programme under grant agreements No. 101070284, No. 10107040 and No. 101189771. RS also acknowledges funding received under Independent Research Fund Denmark (DFF) under grant 4307-00143B. The authors thank members of \href{https://saintslab.github.io/}{SAINTS Lab} for useful discussions.

%% file: appendix_revision.tex
\section{Dataset Condensation in Practice} \label{sec:trick}
\rev{This appendix collects survey background referenced from the main text (generative modeling, loss regularization, and augmentation operators used by DC methods). It contains no new formal claims; the paper's formal contributions are confined to Sections~2--4 of the main text.}

In \Cref{sec:dc}, we introduced a theoretical framework for dataset condensation, which consists of two main steps: matching distributions and optimizing synthetic datasets. Both steps can be performed either in the original high-dimensional input space or within a lower-dimensional latent space. In this section, we explore practical strategies for constructing such latent spaces, as well as technical techniques for enhancing downstream performance and improving the quality of the synthetic data.

\subsection{Generative Modeling} \label{sec:gen}
A generative model typically consists of two components: an encoder $g_e: \mathbb{R}^n \to \mathbb{R}^m$ and a decoder $g_d: \mathbb{R}^m \to \mathbb{R}^n$, where the input dimension $n$ is significantly larger than the latent dimension $m$, i.e., $n \gg m$. Given a data distribution $\mu$ over $\mathbb{R}^n$, the goal of the generative model is to learn a structured latent representation of the data such that we can sample a latent vector $\mathbf{z} \in \mathbb{R}^m$ from the latent space and reconstruct a data point $\mathbf{x} = g_d (\mathbf{z}) \in \mathbb{R}^n$ that approximately lies in the support of $\mu$.

\subsubsection{Variational Autoencoder (VAE) \citep{kingma2013auto}}
A Variational Autoencoder (VAE) consists of an encoder $g_e$ and a decoder $g_d$. The encoder maps input data to a latent distribution, typically matching it to a prior such as a multivariate Gaussian. The decoder reconstructs data samples by mapping latent variables drawn from this prior back to the original data space. This framework enables both representation learning and generative modeling under a probabilistic formulation. Several variants of the VAE have been developed to enhance performance or address specific limitations, such as $\beta$-VAE (for disentangled representations) \citep{higgins2017beta}, Conditional VAE (CVAE, for supervised generation) \citep{sohn2015learning}, InfoVAE \citep{zhao2017infovae}, and Vector Quantized VAE (VQ-VAE) \citep{van2017neural}.

\subsubsection{Generative Adversarial Network (GAN) \citep{goodfellow2014generative}}
A Generative Adversarial Network (GAN) consists of two components: a generator $g_d$, which maps latent vectors from a low-dimensional space to the data space, and a discriminator $g_c: \mathbb{R}^n \to \mathbb{R}$, which aims to distinguish between real samples and those generated by the generator. The two networks are trained in an adversarial setup: the discriminator is optimized to accurately classify real versus fake data, while the generator is trained to produce samples that are indistinguishable from real ones, thereby "fooling" the discriminator. This adversarial training framework leads to a minimax optimization problem. Numerous variants of GANs have been proposed to improve training stability and generation quality, such as Conditional GAN \citep{mirza2014conditional}, WGAN \citep{arjovsky2017wasserstein}, StyleGAN \citep{karras2019style, karras2020analyzing, karras2021alias}, SinGAN \citep{shaham2019singan}, and CycleGAN \citep{zhu2017unpaired}.

\subsubsection{Latent Diffusion Model (LDM) \citep{rombach2022high}}
A Latent Diffusion Model (LDM) is a generative model that, like a VAE, operates in a compressed latent space. However, unlike VAEs, LDMs use a diffusion process to model the data distribution. Specifically, an encoder $g_e$ first maps high-dimensional data to a lower-dimensional latent space ($n \gg m$). A forward diffusion process $g_{df}: \mathbb{R}^m \to \mathbb{R}^m \to \cdots \to \mathbb{R}^m$ then gradually corrupts the latent variables by adding Gaussian noise over multiple time steps, producing a sequence of noisy latent representations. A learned denoising process $g_{dn}$, typically implemented as a U-Net, then reverses this diffusion to recover clean latent samples. Finally, a decoder $g_d$ maps the denoised latent variables back to the data space:
$$g_d \circ g_{dn} \circ g_{df} \circ g_e: \mathbb{R}^n \to \mathbb{R}^n.$$
The encoder and decoder are typically built from convolutional neural networks (CNNs), while the denoiser is trained using a noise prediction objective similar to that in denoising diffusion probabilistic models (DDPMs) \citep{ho2020denoising}. Several variants of LDMs exist, including Stable Diffusion \citep{rombach2022high}, Imagen \citep{saharia2022photorealistic}, and DALL·E 3 \citep{betker2023improving}.

\subsubsection{Hallucinator \citep{liu2022dataset}}
A hallucinator is a type of image generator that transforms a basis vector $\mathbf{x} \in \mathbb{R}^{h \times w \times c}$ into an output image $\mathbf{x}' \in \mathbb{R}^{h' \times w' \times c'}$ through a composition of three components: an encoder, a linear transformation, and a decoder:
$$g_d \circ g_l \circ g_e: \mathbb{R}^{h \times w \times c} \to \mathbb{R}^{h' \times w' \times c'}.$$
Here, the encoder $g_e$ and decoder $g_d$ are typically implemented using convolutional neural networks (CNNs), while the linear transformer $g_l$ performs an affine transformation of the form $g_l (\mathbf{x}) = \sigma \mathbf{x} + \mu$, where $\sigma$ and $\mu$ are learnable parameters. Importantly, the shape of the input basis vector $\mathbf{x}$ need not match the shape of the generated image $\mathbf{x}'$, which allows for flexible dimensionality and structured generation.

\subsubsection{Tensor decomposition \citep{yang2024neural}} \label{sec:tensor}
A tensor decomposition operator maps a high-dimensional tensor (e.g. a 4D-tensor for images of standard size $b \times c \times h \times w$, where $b, c, h, w$ represent batch size, channels, height, and width respectively) into a pair of tensors: a spectral tensor $\mathbf{T}$ and a kernel tensor $\mathbf{K}$,
$$\mathcal{A}: \mathbb{R}^{n_1 \times n_2 \times \cdots \times n_d} \to \mathbb{R}^{n_1' \times n_2' \times \cdots \times n_d'} \times \mathbb{R}^{n_1' \times n_2' \times \cdots \times n_d' \times n_1 \times n_2 \times \cdots \times n_d}, \; \mathbf{X} \mapsto (\mathbf{T}, \mathbf{K}),$$
such that the original tensor $\mathbf{X}$ can be reconstructed by the tensor product of these components $\mathbf{X} = \mathbf{T} \otimes \mathbf{K}$. We further assume the kernel tensor $\mathbf{K}$ is separable, meaning it can be expressed as a tensor product of lower-dimensional matrices $\mathbf{K} = \hat{\mathbf{K}}_1 \otimes \cdots \otimes \hat{\mathbf{K}}_d$, where each $\hat{\mathbf{K}}_i \in \mathbb{R}^{n_i' \times n_i}$ is a 2D tensor. The original tensor $\mathbf{X}$ contains $\prod_{i = 1}^d n_i$ elements. In contrast, storing $(\mathbf{T}, \mathbf{K})$ requires only $\prod_{i = 1}^d n_i' + \sum_{i = 1}^d n_i n_i' = O(\sum_{i = 1}^d n_i)$ elements, since in practice the dimensions are chosen such that $n_i' \ll n_i$ in practice, resulting in a substantial reduction in storage.

\subsection{Loss Regularization} \label{sec:loss_reg}
The loss landscape of a neural network's training objective is typically non-convex and highly complex, making it challenging to find a globally optimal solution. To address this, various loss regularization terms are commonly introduced in practice. Depending on the desired property, these regularizations can serve different purposes: accelerating convergence, preventing overfitting, or encouraging additional diversity or similarity within the synthetic dataset.

\subsubsection{Diversity term \citep{zhang2023dataset}}
Let $\mathcal{S} \subseteq \mathbb{R}^n$ denote a synthetic dataset, and let $\mathcal{S}^y \subseteq \mathcal{S}$ be the subset of samples labeled with class $y$. To balance intra-class diversity and inter-class discriminability, while also mitigating overfitting during training on synthetic data, we introduce the intra-class diversity loss:
$$L_{intra} (\mathbf{x}_i^y, \tau) = - \log \frac{e^{\langle h(\mathbf{x}_i^y), c(y)\rangle / \tau}}{e^{\langle h(\mathbf{x}_i^y), c(y)\rangle / \tau} + \sum_{i \ne j} e^{\langle h(\mathbf{x}_i^y), h(\mathbf{x}_j^y)\rangle / \tau}},$$
where $\mathbf{x}_i^y \in \mathcal{S}^y$ is a synthetic sample with label $y$, $c(y)$ is a learned class embedding for label $y$, $h(\cdot)$ denotes the feature extractor (e.g., a neural network encoder), and $\tau > 0$ is a temperature hyperparameter controlling the sharpness of similarity scores. Minimizing this loss promotes two effects: (1) it encourages a higher similarity (inner product) between the sample representation $h(\mathbf{x}_i^y)$ and its corresponding class embedding $c(y)$, reinforcing correct class prediction, and (2) it penalizes high similarity between different samples $\mathbf{x}_i^y$ and $\mathbf{x}_j^y$ ($i \ne j$) of the same class, encouraging feature diversity within class $y$. This objective helps ensure that synthetic samples are both representative and diverse, improving generalization while maintaining class consistency.

Complementary to the intra-class diversity objective, we introduce the inter-class diversity loss to promote separation between synthetic samples from different classes. It is defined as:
$$L_{inter} (\mathcal{S}, \tau) = \sum_{y_1 \ne y_2} \text{ReLU} \big(\tau - \big\|\overline{h(\mathcal{S}^{y_1})} - \overline{h(\mathcal{S}^{y_2})}\big\|_2\big),$$
where $\tau > 0$ is a margin (threshold) hyperparameter. This loss penalizes class pairs whose mean feature embeddings are too close, i.e., when their Euclidean distance falls below the threshold $\tau$. By minimizing $L_{\text{inter}}$, the feature centroids of different classes are encouraged to be at least $\tau$ apart, promoting inter-class separability and reducing class confusion in the learned representation space. If the difference between the mean features of two different classes exceeds the margin, then the penalty to this class pair is zero. 

\subsubsection{Representative term \citep{gu2024efficient}}
To ensure that the small-scale synthetic dataset $\mathcal{S}$ effectively captures the characteristics of the original dataset $\mathcal{T}$, we introduce a representativeness loss that encourages each synthetic sample to align closely with at least one real data point. For a synthetic sample $\mathbf{x}_1 \in \mathcal{S}$, the loss is defined as:
$$L_{\text{rep}} (\mathbf{x}_1) = - \min_{\mathbf{x}_2 \in \mathcal{T}} \sigma (\mathbf{x}_1, \mathbf{x}_2),$$
where $\sigma(\mathbf{x}_1, \mathbf{x}_2)$ denotes the cosine similarity $\sigma (\mathbf{x}_1, \mathbf{x}_2) = \frac{\langle \mathbf{x}_1, \mathbf{x}_2 \rangle}{\|\mathbf{x}_1\|_2 \|\mathbf{x}_2\|_2}$. Minimizing $L_{\text{rep}}$ encourages each synthetic point to closely resemble at least one sample in the original dataset, improving its representational fidelity. To further enhance coverage, we introduce a diversity loss that discourages redundancy among synthetic samples. For a given $\mathbf{x}_1 \in \mathcal{S}$, the loss is defined as:
$$L_{\text{div}} (\mathbf{x}_1) = \max_{\mathbf{x}_2 \in \mathcal{S}, \mathbf{x}_1 \ne \mathbf{x}_2} \sigma (\mathbf{x}_1, \mathbf{x}_2).$$
Minimizing $L_{\text{div}}$ pushes synthetic samples away from each other in feature space, promoting diversity within $\mathcal{S}$. Together, minimizing both $L_{\text{rep}}$ and $L_{\text{div}}$ ensures that each synthetic sample is representative of the original data while remaining distinct from other synthetic points, striking a balance between coverage and diversity.

\subsubsection{Contrastive term \citep{liu2022dataset}}
Let $\mathcal{H}$ be a finite set of hypotheses (e.g., neural network models). To evaluate how consistently synthetic data features are represented across different models within the same class, we introduce the adversarial contrastive loss:
$$L_{con} (\mathcal{H}, \mathcal{S}^y, \tau) = - \frac{1}{|\mathcal{H}|^2} \frac{1}{|\mathcal{S}^y|} \sum_{j \ne k} \sum_{i = 1}^{|\mathcal{S}^y|} \log \frac{e^{\langle h_j (\mathbf{x}_i^y), h_k (\mathbf{x}_i^y)\rangle / \tau}}{\sum_{t = 1}^{|\mathcal{S}^y|} e^{\langle h_j (\mathbf{x}_i^y), h_k (\mathbf{x}_t^y)\rangle / \tau}}.$$
where $h_j, h_k \in \mathcal{H}$ are hypotheses, $\mathbf{x}_i^y \in \mathcal{S}^y$ are synthetic samples from class $y$, and $\tau$ is a temperature scaling hyperparameter. This loss encourages consistency across models by aligning features extracted by different hypotheses for the same input, while contrasting against features from other class-$y$ samples. In parallel, to quantify feature diversity within class $y$ across models, we define the cosine similarity loss:
$$L_{cos} (\mathcal{H}, \mathcal{S}^y, \tau) = \frac{1}{|\mathcal{H}|^2} \frac{1}{|\mathcal{S}^y|} \sum_{j \ne k} \sum_{i = 1}^{|\mathcal{S}^y|} \sigma (h_j (\mathbf{x}_i^y), h_k (\mathbf{x}_i^y)),$$
where $\sigma(\cdot,\cdot)$ denotes the cosine similarity $\sigma (\mathbf{x}_1, \mathbf{x}_2) = \frac{\langle \mathbf{x}_1, \mathbf{x}_2 \rangle}{\|\mathbf{x}_1\|_2 \|\mathbf{x}_2\|_2}$. By minimizing both $L_{\text{con}}$ and $L_{\text{cos}}$, we enforce a dual objective: (1) inter-model alignment of features within each class to ensure semantic consistency, and (2) intra-class variation to maintain feature diversity, avoiding degenerate or overly uniform representations across the hypothesis ensemble.

\subsubsection{Discrimination term \citep{wang2022cafe}}
Let $\mathcal{T}$ be a finite dataset, $\mathcal{S}$ a synthetic dataset, and $h$ a trained model. For each class label $y \in {1, \dots, C}$, denote by $\overline{h(\mathcal{S}^y)}$ the mean feature representation of synthetic samples from class $y$. For each sample $\mathbf{x}_i^y \in \mathcal{T}^y$ in the real dataset, define the class-similarity score vector $\mathbf{O}_i^y = [\langle h(\mathbf{x}_i^y), \overline{h(\mathcal{S}^y)} \rangle]_{y = 1}^C \in \mathbb{R}^C$. Intuitively, $\mathbf{O}_i^y$ encodes the similarity between the feature of a real sample and the average synthetic features from each class. We then define the discrimination loss as:
$$L_{dis} = - \frac{1}{C} \frac{1}{|\mathcal{T}^y|} \sum_{y = 1}^C \sum_{i = 1}^{|\mathcal{T}^y|} \log c (\mathbf{O}_i^y),$$
where $c (\mathbf{O}_i^y)_y$ denotes the softmax probability assigned to the true class $y$. This loss measures how well the average synthetic class features serve as discriminative class prototypes for real data. By minimizing $L_{\text{dis}}$, the synthetic dataset $\mathcal{S}$ is encouraged to produce class-wise features that closely align with the features extracted from real samples in $\mathcal{T}$, thereby improving the semantic alignment between synthetic and real data.

\subsection{Data Augmentation} \label{sec:aug}
While generative modeling seeks to reduce the dimensionality of the data space, a complementary technique is data augmentation, which aims to increase the diversity or complexity of the data. The goal is to improve downstream performance by training models on augmented data that better captures the variability of the original distribution.

Given a 4D tensor representing a batch of images of shape $b \times c \times h \times w$, where $b$, $c$, $h$, and $w$ denote the batch size, number of channels, height, and width respectively, the multi-formation operator \citep{kim2022dataset} increases the number of channels by a factor of $r$, referred to as the formation factor:
$$\mathcal{A}_r: \mathbb{R}^{b \times c \times h \times w} \to \mathbb{R}^{b \times c(r^2+1) \times h \times w}.$$
The illustration of multi-formation process is shown in \Cref{fig:multi-formation}.
\begin{figure}[h]
    \centering
    \includegraphics[width=\linewidth]{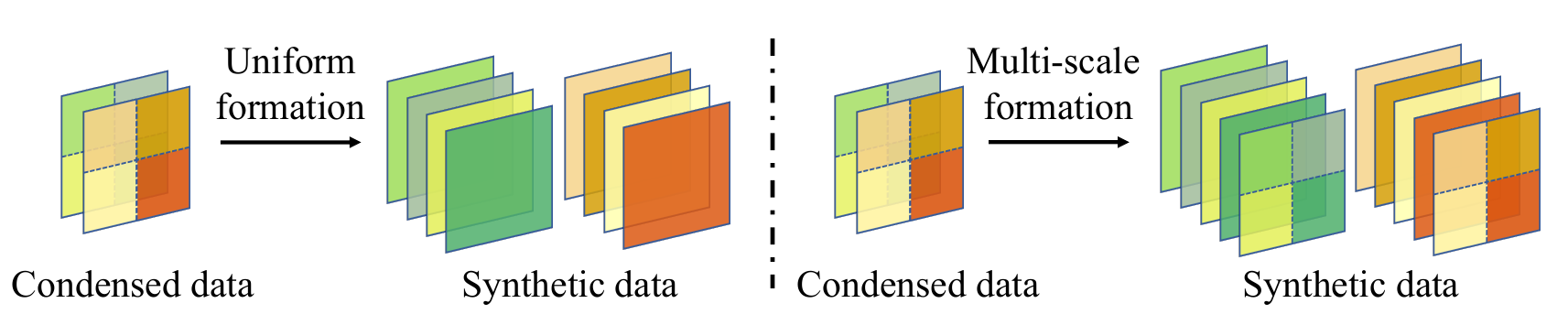}
    \caption{Illustration of multi-formation operator with a factor of 2.}
    \label{fig:multi-formation}
\end{figure}

In recent work, \citep{zhou2024dataset} propose a channel-wise multi-formation operator for data augmentation, defined as
$$\mathcal{A}_c: \mathbb{R}^{b \times c \times h \times w} \to \mathbb{R}^{4 b \times c \times h \times w},$$
which expands the dataset by a factor of four. Specifically, each input sample is augmented using three $1 \times 1$ convolutional layers, each acting as a color mapping function. As a result, a batch of $b$ samples is transformed into $3b$ additional color-mapped samples via $3b$ distinct convolution layers. The original and augmented (color-shifted) data are then concatenated along the batch dimension, and the resulting tensor can be used for model training or as input to dataset condensation pipelines.

\rev{\section{Full Results of the Empirical Study} \label{sec:app_exp}
This appendix reports the complete per-method results of the controlled study of \Cref{sec:exp}; the main text summarizes them in Figure~\ref{fig:exp} and the accompanying discussion.}

\begin{table}[h]
\color{black}
\caption{Full results of the empirical illustration of \Cref{sec:exp}: CIFAR-10, ConvNet, 5 seeds (mean $\pm$ std). Acc.: test accuracy after retraining from scratch on the condensed set; Gap: empirical transfer gap $\mathrm{GD}_{\mathcal{A}}$ vs.\ the same-seed full-data reference; $\widehat{\mathrm{MMD}}^2_{k_p}$: induced random-feature kernel MMD; $\mathrm{PD}_{\mathcal{G}}$: coupled, permutation-aligned parameter distance; GPU-h: condensation cost. TM: trajectory matching.}
\label{tab:exp}\centering
\begin{tabular}{llcccccc}
\toprule
Level & Method & IPC & Acc.\ (\%) & Gap (\%) & $\widehat{\mathrm{MMD}}^2_{k_p}$ & $\mathrm{PD}_{\mathcal{G}}$ & GPU-h \\
\midrule
\multirow{3}{*}{GD} & \multirow{3}{*}{RaT-BPTT} & 1 & $22.1 \pm 2.3$ & $62.1 \pm 2.4$ & $287.54 \pm 2.76$ & $15.7 \pm 0.1$ & $0.4$ \\
 & & 10 & $54.4 \pm 0.5$ & $29.8 \pm 0.4$ & $22.11 \pm 1.25$ & $14.3 \pm 0.2$ & $1.1$ \\
 & & 50 & $60.6 \pm 0.3$ & $23.6 \pm 0.5$ & $14.37 \pm 0.16$ & $13.1 \pm 0.3$ & $1.2$ \\
\midrule
\multirow{3}{*}{VD} & \multirow{3}{*}{CAFE} & 1 & $23.4 \pm 0.7$ & $60.8 \pm 0.6$ & $251.39 \pm 10.19$ & $15.7 \pm 0.2$ & $0.3$ \\
 & & 10 & $45.6 \pm 0.4$ & $38.6 \pm 0.5$ & $4.83 \pm 0.02$ & $14.7 \pm 0.1$ & $0.4$ \\
 & & 50 & $57.5 \pm 0.7$ & $26.7 \pm 0.5$ & $1.22 \pm 0.01$ & $13.7 \pm 0.2$ & $0.9$ \\
\midrule
\multirow{3}{*}{PD} & \multirow{3}{*}{TM} & 1 & $20.0 \pm 1.9$ & $64.3 \pm 2.0$ & $247.67 \pm 5.03$ & $15.9 \pm 0.2$ & $0.4$ \\
 & & 10 & $35.5 \pm 2.2$ & $48.7 \pm 1.9$ & $17.03 \pm 1.02$ & $14.9 \pm 0.2$ & $1.0$ \\
 & & 50 & $50.7 \pm 0.9$ & $33.5 \pm 0.8$ & $4.25 \pm 0.25$ & $14.2 \pm 0.3$ & $2.3$ \\
\midrule
\multirow{6}{*}{DD} & \multirow{3}{*}{DM} & 1 & $26.2 \pm 0.3$ & $58.1 \pm 0.6$ & $51.61 \pm 1.11$ & $15.5 \pm 0.1$ & $0.5$ \\
 & & 10 & $46.9 \pm 0.8$ & $37.3 \pm 1.0$ & $3.13 \pm 0.01$ & $14.5 \pm 0.2$ & $0.5$ \\
 & & 50 & $61.5 \pm 0.6$ & $22.7 \pm 0.7$ & $0.97 \pm 0.00$ & $13.4 \pm 0.4$ & $0.7$ \\
\cmidrule(lr){2-8}
 & \multirow{3}{*}{M3D} & 1 & $27.6 \pm 1.0$ & $56.6 \pm 1.1$ & $79.97 \pm 2.40$ & $15.4 \pm 0.2$ & $0.7$ \\
 & & 10 & $41.1 \pm 0.7$ & $43.1 \pm 0.7$ & $9.90 \pm 0.32$ & $14.8 \pm 0.2$ & $0.7$ \\
 & & 50 & $51.0 \pm 0.8$ & $33.2 \pm 0.6$ & $3.98 \pm 0.20$ & $14.1 \pm 0.4$ & $0.8$ \\
\bottomrule
\end{tabular}\end{table}